\documentclass[twoside,11pt]{article}

\usepackage{jmlr2e}

\usepackage{amsbsy,amsfonts,dsfont,units,amsmath}
\usepackage{algorithm,algorithmic}
\usepackage{color,graphics,stmaryrd}
\usepackage{hyperref}
\usepackage{tikz}
\usetikzlibrary{arrows}
\usetikzlibrary{calc}
\usepackage{tikz-3dplot}

\newtheorem{condition}{Condition}
\newtheorem{assumption}{Assumption}

\providecommand{\eqref}{} 
\renewcommand{\eqref}[1]{Eq.~(\ref{eq:#1})}
\newcommand{\figref}[1]{Figure~\ref{fig:#1}}
\newcommand{\tabref}[1]{Table~\ref{table:#1}}        
\newcommand{\secref}[1]{Section~\ref{sec:#1}}
\newcommand{\thmref}[1]{Theorem~\ref{thm:#1}}
\newcommand{\lemref}[1]{Lemma~\ref{lem:#1}}

\newcommand{\corref}[1]{Corollary~\ref{cor:#1}}
\newcommand{\appref}[1]{Appendix~\ref{app:#1}}
\newcommand{\propref}[1]{Proposition~\ref{prop:#1}}
\newcommand{\algref}[1]{Algorithm~\ref{alg:#1}}
\newcommand{\condref}[1]{Condition~\ref{cond:#1}}
\newcommand{\defref}[1]{Def.~\ref{def:#1}}
\newcommand{\asref}[1]{Assumption~\ref{as:#1}}

\renewcommand{\P}{\ensuremath{\mathbb{P}}}
\newcommand{\E}{\ensuremath{\mathbb{E}}}

\newcommand{\tr}{\ensuremath{\operatorname{tr}}}

\newcommand{\reals}{\ensuremath{\mathbb{R}}}
\newcommand{\nats}{\ensuremath{\mathbb{N}}}
\newcommand{\half}{\ensuremath{\frac12}}
\newcommand{\ceil}[1]{\ensuremath{\lceil #1\rceil}}
\newcommand{\Ceil}[1]{\ensuremath{\left\lceil #1\right\rceil}}
\newcommand{\floor}[1]{\ensuremath{\lfloor #1\rfloor}}

\newcommand{\dotprod}[1]{\ensuremath{\langle #1 \rangle}}

\newcommand{\grad}{\ensuremath{\nabla}}

\DeclareMathOperator*{\argmin}{argmin}

\newcommand\opt{{\ensuremath{\operatorname{\star}}}}
\newcommand\sq{{\ensuremath{\operatorname{sq}}}}

\newcommand{\cD}{\ensuremath{\mathcal{D}}}
\newcommand{\cE}{\ensuremath{\mathcal{E}}}

\newcommand{\cN}{\ensuremath{\mathcal{N}}}

\newcommand{\cX}{\ensuremath{\mathcal{X}}}
\newcommand{\cY}{\ensuremath{\mathcal{Y}}}
\newcommand{\cZ}{\ensuremath{\mathcal{Z}}}

\newcommand\X{\ensuremath{\mathbb{X}}}
\renewcommand\t{\ensuremath{{\scriptscriptstyle{\top}}}}
\newcommand\ind[1]{{\ensuremath{\mathds{1}\{#1\}}}}
\newcommand\veps{{\ensuremath{\varepsilon}}}
\newcommand\eps{{\ensuremath{\epsilon}}}
\newcommand\Sig{{\ensuremath{\varSigma}}}
\newcommand\ball{B}
\renewcommand\vec[1]{{\boldsymbol{#1}}}
\newcommand\ve{{\vec{e}}}
\newcommand\vw{{\vec{w}}}
\newcommand\vx{{\vec{x}}}
\newcommand\vu{{\vec{u}}}
\newcommand\vX{{\vec{X}}}
\newcommand\vy{{\vec{y}}}
\newcommand\vv{{\vec{v}}}
\newcommand\Id{\ensuremath{\operatorname{Id}}}
\newcommand{\dotp}[1]{\ensuremath{\langle #1 \rangle_{\X}}}
\newcommand\supp{\operatorname{supp}}

\newcommand\norm[1]{\ensuremath{\|#1\|}}
\newcommand\dualnorm[1]{\ensuremath{\|#1\|_*}}

\newcommand{\nrm}{{\norm{\cdot}}}
\newcommand{\dualnrm}{{\norm{\cdot}_*}}

\newcommand\wh{\widehat}
\newcommand\bbeta{\ensuremath{\bar{\beta}}}
\newcommand\emp{\ensuremath{\operatorname{emp}}}

\newcommand\median{\ensuremath{\operatorname{median}}}
\newcommand\rank{\ensuremath{\operatorname{rank}}}

\newcommand\APPROX{\ensuremath{\operatorname{APPROX}}}
\newcommand\DIST{\ensuremath{\operatorname{DIST}}}
\newcommand{\gmed}{\mathrm{sumd}}
\newcommand{\omed}{\Delta}

\newcommand{\todo}[1]{}

\jmlrheading{17}{2016}{1-40}{7/14; Revised 6/15}{4/16}{Daniel Hsu and Sivan Sabato}

\ShortHeadings{Loss minimization and parameter estimation with heavy tails}{Hsu and Sabato}
\firstpageno{1}

\begin{document}

\title{Loss Minimization and Parameter Estimation \\ with Heavy Tails}

\author{%
\name Daniel Hsu
  \email djhsu@cs.columbia.edu \\
  \addr Department of Computer Science \\
  Columbia University \\
  New York, NY 10027, USA
       \AND
       \name Sivan Sabato \email sabatos@cs.bgu.ac.il \\
       \addr Department of Computer Science\\
       Ben-Gurion University of the Negev\\
       Beer-Sheva 8410501, Israel\\
       }

\editor{David Dunson}

\maketitle

\begin{abstract}%
This work studies applications and generalizations of a simple
estimation technique that provides exponential concentration under
heavy-tailed distributions, assuming only bounded low-order moments.
We show that the technique can be used for approximate minimization of
smooth and strongly convex losses, and specifically for least squares
linear regression.
For instance, our $d$-dimensional estimator requires just
$\tilde{O}(d\log(1/\delta))$ random samples to obtain a constant
factor approximation to the optimal least squares loss with
probability $1-\delta$, without requiring the covariates or noise to
be bounded or subgaussian.
We provide further applications to sparse linear regression and
low-rank covariance matrix estimation with similar allowances on the
noise and covariate distributions.
The core technique is a generalization of the median-of-means estimator
to arbitrary metric spaces.%
\end{abstract}


\section{Introduction}
\label{sec:intro}

The minimax principle in statistical estimation prescribes procedures
(\emph{i.e.}, estimators) that minimize the worst-case risk over a
large class of distributions generating the data.
For a given loss function, the risk is the expectation of the loss of
the estimator, where the expectation is taken over the data examined
by the estimator.
For example, for a large class of loss functions including squared
loss, the empirical mean estimator minimizes the worst-case risk over
the class of Gaussian distributions with known
variance~\citep{wolfowitz50}.
In fact, Gaussian distributions with the specified variance are
essentially the worst-case family of distributions for squared loss,
at least up to constants~\citep[see, \emph{e.g.},][Proposition
6.1]{catoni}.

In this work, we are interested in estimators whose deviations from
expected behavior are controlled with very high probability over the
random draw of the data examined by the estimator.
Deviations of the behavior of the estimator from its expected behavior are worrisome especially when data come from unbounded
and/or heavy-tail distributions, where only very low order moments may
be finite.
For example, the Pareto distributions with shape parameter $\alpha>0$
are unbounded and have finite moments only up to orders $< \alpha$;
these distributions are commonly associated with the modeling of
extreme events that manifest in data.
Bounds on the expected behavior of an estimator are insufficient in
these cases, since the high-probability guarantees that may be derived from such bounds
(say, using Markov's inequality) are rather weak.
For example, if the risk (\emph{i.e.}, expected loss) of an estimator
is bounded by $\epsilon$, then all that we may derive from Markov's
inequality is that the loss is no more than $\epsilon/\delta$ with
probability at least $1-\delta$.
For small values of $\delta \in (0,1)$, the guarantee is not very
reassuring, but it may be all one can hope for in these extreme
scenarios---see Remark~\ref{remark:empirical-mean} in
\secref{median-of-means} for an example where this is tight.
Much of the work in statistical learning theory is also primarily
concerned with such high probability guarantees, but the bulk of the
work makes either boundedness or subgaussian tail assumptions that
severely limit the applicability of the results even in settings as
simple as linear regression~\citep[see,
\emph{e.g.},][]{smooth-loss,ohad}.

Recently, it has been shown that it is possible to improve on methods
which are optimal for expected behavior but suboptimal when
high-probability deviations are concerned~\citep{AC11,catoni,lugosi}.
These improvements, which are important when dealing with heavy-tailed
distributions, suggest that new techniques (\emph{e.g.}, beyond
empirical risk minimization) may be able to remove the reliance on
boundedness or control of high-order moments.
\citet{BubeckCeLu13} show how a more robust mean estimator can be used
for solving the stochastic multi-armed bandit problem under
heavy-tailed distributions.

This work applies and generalizes a technique for controlling large
deviations from the expected behavior with high probability, assuming only
bounded low-order moments such as variances.
We show that the technique is applicable to minimization of smooth and
strongly convex losses, and derive specific loss bounds for least
squares linear regression, which match existing rates, but without
requiring the noise or covariates to be bounded or subgaussian.
This contrasts with several recent
works~\citep{smooth-loss,HKZ12,ohad} concerned with (possibly
regularized) empirical risk minimizers that require such assumptions.
It is notable that in finite dimensions, our result implies that a constant
factor approximation to the optimal loss can be achieved with a sample size
that is independent of the size of the optimal loss.
This improves over the recent work of~\citet{MJ13}, which has a
logarithmic dependence on the optimal loss, as well as a suboptimal
dependence on specific problem parameters (namely condition numbers).
We also provide a new generalization of the basic technique for
general metric spaces, which we apply to least squares linear
regression with heavy tail covariate and noise distributions, yielding
an improvement over the computationally expensive procedure
of~\citet{AC11}.

The basic technique, found in the textbook of \citet[p.~243]{NY}, is
very simple, and can be viewed as a generalization of the
median-of-means estimator used by \citet{alon99} and many others.
The idea is to repeat an estimate several times, by splitting the
sample into several groups, and then selecting a single estimator out of 
the resulting list of candidates.
If an estimator from one group is good with noticeably
better-than-fair chance, then the selected estimator will be good with
probability exponentially close to one.
This is remininscant of techniques from \emph{robust
statistics}~\citep{Huber81}, although our aim is expressly different
in that our aim is good performance on the same probability
distribution generating the data, rather than an uncontaminated or
otherwise better behaved distribution.
Our new technique can be cast as a simple selection problem in general
metric spaces that generalizes the scalar median.

We demonstrate the versatility of our technique by giving further
examples in sparse linear regression~\citep{tibshirani96} under
heavy-tailed noise and low-rank covariance covariance matrix
approximation~\citep{KLT11} under heavy-tailed covariate
distributions.
We also show that for prediction problems where there may not be a
reasonable metric on the predictors, one can achieve similar
high-probability guarantees by using median aggregation in the output
space.

The initial version of this article~\citep{HS13-heavy,HsuSabato14}
appeared concurrently with the simultaneous and independent work
of~\citet{Minsker13}, which develops a different generalization of the
median-of-means estimator for Banach and Hilbert spaces.
We provide a new analysis and comparison of this technique to ours in
\secref{geometric}.
We have also since become aware of the earlier work by~\citet{LO11},
which applies the median-of-means technique to empirical risks in
various settings much like the way we do in \algref{main2}, although
our metric formulation is more general.
Finally, the recent work of~\citet{lugosi} vastly generalizes the
techniques of~\citet{catoni} to apply to much more general settings,
although they retain some of the same deficiencies (such as the need
to know the noise variance for the optimal bound for least squares
regression), and hence their results are not directly comparable to
ours.

\todo{general discussion /related works about heavy-tails, robust
statistics (caltech book?), minsker's results. check also minsker's new work}



\section{Overview of Main Results}\label{sec:overview}

This section gives an overview of the main results.

\subsection{Preliminaries}

Let $[n] := \{1,2,\dotsc,n\}$ for any natural number $n \in \nats$.
Let $\ind{P}$ take value $1$ if the predicate $P$ is true, and $0$
otherwise. Assume an example space $\cZ$, and a distribution $\cD$ over the space. Further assume a space of predictors or estimators $\X$. We consider learning or estimation algorithms that accept as input an i.i.d.\ sample of size $n$ drawn from $\cD$ and a confidence parameter $\delta \in (0,1)$, and return an estimator (or predictor) $\hat{\vw} \in \X$.  
For a (pseudo) metric $\rho$ on $\X$, let $\ball_\rho(\vw_0,r) := \{
\vw \in \X : \rho(\vw_0,\vw) \leq r \}$ denote the ball of radius $r$
around $\vw_0$.

We assume a loss function $\ell:\cZ \times \X \rightarrow \reals_+$ that assigns a non-negative number to a pair of an example from $\cZ$ and a predictor from $\X$, and consider the task of finding a predictor that has a small loss in expectation over the distribution of data points, based on an input sample of $n$ examples drawn independently from $\cD$. The expected loss of a predictor $\vw$ on the distribution is denoted $L(\vw) = \E_{Z \sim D}(\ell(Z,\vw))$. Let $L_\opt := \inf_{\vw} L(\vw)$. Our goal is to find $\hat{\vw}$ such that $L(\hat{\vw})$ is close to $L_\opt$. 

In this work, we are interested in performance guarantees that hold
with high probability over the random draw of the input sample and
any internal randomization used by the estimation algorithm.
Thus, for a given allowed probability of failure $\delta \in (0,1)$,
we study excess loss $L(\hat{\vw}) - L_\opt$ achieved by the predictor
$\hat{\vw} \equiv \hat{\vw}(\delta)$ returned by the algorithm on a
$1-\delta$ probability subset of the sample space.
Ideally, the excess loss only depends sub-logarithmically on
$1/\delta$, which is the dependence achieved when the distribution of
the excess loss has exponentially decreasing tails.
Note that we assume that the value of $\delta$ is provided as input to
the estimation algorithm, and only demand the probabilistic guarantee
for this given value of $\delta$.
Therefore, strictly speaking, the excess loss need not exhibit
exponential concentration.
Nevertheless, in this article, we shall say that an estimation
algorithm achieves exponential concentration whenever it guarantees,
on input $\delta$, an excess loss that grows only as $\log(1/\delta)$.

%
%


\subsection{Robust Distance Approximation}

Consider an estimation problem, where the goal is to estimate an unknown parameter of the distribution, using a random i.i.d.~sample from that distribution. We show throughout this work that for many estimation problems, if the sample is split into non-overlapping subsamples, and estimators are obtained independently from each subsample, then with high probability, this generates a set of estimators such that some fraction of them are close, under a meaningful metric, to the true, unknown value of the estimated parameter. Importantly, this can be guaranteed in many cases even under under heavy-tailed distributions.

Having obtained a set of estimators, a fraction of which are close to the estimated parameter, the goal is now to find a single good estimator based on this set. This goal is captured by the following general problem, which we term \emph{Robust Distance Approximation}. 
A Robust Distance Approximation procedure is given a set of points in a metric space and returns a single point from the space. This single point should satisfy the following condition: If there is an element in the metric space that a certain fraction of the points in the set are close to, then the output point should also be close to the same element. 
Formally, let $(\X,\rho)$ be a metric space. Let $W \subseteq \X$ be a (multi)set of size $k$ and let $w_\opt$ be a distinguished element in $\X$. 
For $\alpha \in (0,\half)$ and $w \in \X$, denote by $\omed_W(w,\alpha)$ the minimal number $r$ such that $|\{v \in W \mid \rho(w,v) \leq r\}| > k(\half + \alpha)$. We often omit the subscript $W$ and write simply $\omed$ when $W$ is known.

We define the following problem:
\begin{definition}[Robust Distance Approximation]\label{def:rda}
Fix $\alpha \in (0,\half)$. Given $W$ and $(\X,\rho)$ as input, return $y \in \X$ such that 
$\rho(y,w_\opt) \leq C_\alpha\cdot\omed_W(w_\opt,\alpha)$, for some constant $C_\alpha \geq 0$. $C_\alpha$ is the \emph{approximation factor} of the procedure.
\end{definition}

In some cases, learning with heavy-tailed distributions requires using
a metric that depends on the distribution. Then, the Robust Distance
Estimation procedure has access only to noisy measurements of
distances in the metric space, and is required to succeed with high
probability. In \secref{core} we formalize these notions, and provide
simple implementations of Robust Distance Approximation for general
metric spaces, with and without direct access to the metric. For the
case of direct access to the metric our formulation is similar to that
of \citet{NY}.

\subsection{Convex Loss Minimization}

The general approach to estimation described above has many applications.
We give here the general form of our main results for applications, and defer the technical definitions and results to the relevant sections. Detailed discussion of related work for each application is also provided in the appropriate sections.

First, we consider smooth and convex losses. 
We assume that the parameter space $\X$ is a Banach space with a norm
$\nrm$ and a dual norm $\nrm_*$. We prove the following
result:\footnote{Formal definitions of terms used in the conditions are given in \secref{approx}.}
\begin{theorem}\label{thm:lossmain}
There exists an algorithm that accepts as input an i.i.d.\ sample of size $n$ drawn from $\cD$ and a confidence parameter $\delta \in (0,1)$, and returns $\hat{\vw} \in \X$, such that if the following conditions hold:
\begin{itemize}
  \item the dual norm $\nrm_*$ is $\gamma$-smooth;

\item there exists $\alpha > 0$ and sample size $n_\alpha$ such that,
  with probability at least $1/2$, the empirical loss $\vw \mapsto
  \hat{L}(\vw)$ is $\alpha$-strongly convex with respect to $\nrm$
  whenever the sample is of size at least $n_\alpha$;

\item $n \geq C\log(1/\delta)\cdot n_\alpha$ for some universal constant $C > 0$;
\item $\vw \mapsto \ell(z,\vw)$ is $\beta$-smooth with respect to
$\nrm$ for all $z \in \cZ$;

\item $\vw \mapsto L(\vw)$ is $\bbeta$-smooth with respect to
$\nrm$;

\end{itemize}
then with probability at least $1-\delta$, for another universal constant $C' > 0$,
\begin{equation*}
L(\hat\vw) \leq \biggl( 1 +
\frac{C'\beta\bbeta\gamma\lceil\log(1/\delta)\rceil}
{n\alpha^2}
\biggr) L_\opt.
\end{equation*}
\end{theorem}
This gives a constant approximation of the optimal loss with a number of samples that does not depend on the value of the optimal loss. 
The full results for smooth convex losses are provided in \secref{approx}. \thmref{lossmain} is stated in full as \corref{loss}, and we further provide a result with more relaxed smoothness requirements. As apparent in the result, the only requirements on the distribution are those that are implied by the strong convexity and smoothness parameters. This allows support for fairly general heavy-tailed distributions, as we show below.
\todo{related work on smooth and convex losses?}

\subsection{Least Squares Linear Regression}

A concrete application of our analysis of smooth convex losses is linear regression. In linear regression, $\X$ is a Hilbert space with an inner product $\dotp{\cdot,\cdot}$, and it is both the data space and the parameter space. The loss $\ell \equiv \ell^\sq$ is the squared loss \[
\ell^\sq((\vx,y),\vw) := \frac12(\vx^\t \vw - y)^2.
\]
$L^\sq$ and $L^\sq_\opt$ are defined similarly to $L$ and $L_\opt$.

Unlike standard high-probability bounds for regression,
we give bounds that make no assumption on the range or the tails of
the distribution of the response variables, other than a trivial requirement that the optimal squared loss be finite. The assumptions on the distribution of the covariates are also minimal.

Let $\Sig$ be the second-moment operator $\vec{a} \mapsto \E(\vX \dotp{\vX,\vec{a}})$, where $\vX$ is a random data point from the marginal distribution of $\cD$ on $\X$. For a finite-dimensional $\X$, $\Sig$ is simply the (uncentered) covariance matrix $\E[\vX \vX^\t]$.
First, consider the finite-dimensional case, where $\X = \reals^d$, and assume $\Sig$ is not singular. 
Let $\norm{\cdot}_2$ denote the Euclidean norm in $\reals^d$. 
Under only bounded $4+\epsilon$ moments of the marginal on $\X$ 
(a condition that we specify in full detail in \secref{regression}), 
we show the following guarantee.

\begin{theorem}\label{thm:heavyxmain}
Assume the marginal of $\X$ has bounded $4+\epsilon$ moments. There is a constant $C > 0$ and an algorithm that accepts as input a sample of size $n$ and a confidence parameter $\delta \in (0,1)$, and returns $\hat{\vw} \in \X$, such that if $n \geq C d \log(1/\delta)$, with probability at least $1-\delta$,
\[
L^\sq(\hat\vw) \leq
L^\sq_\opt 
+ O\left(\frac{ \E(\norm{\Sig^{-1/2}\vX(\vX^\t\vw_\opt-Y)}_2^2) \log(1/\delta)} {n}\right) .
\]
\end{theorem}
This theorem is stated in full as \thmref{heavyx} in \secref{regression}.
Under standard finite fourth-moment conditions, this result translates to
the bound
\[ 
L^\sq(\hat\vw) \leq \biggl( 1 + O\biggl( \frac{d \log(1/\delta)}n
\biggr) \biggr) L^\sq_\opt,
\]
with probability $\geq 1-\delta$. These results improve over recent
results by~\citet{AC11}, \citet{catoni}, and \citet{MJ13}. We provide
a full comparison to related work in \secref{regression}.

\thmref{heavyxmain} can be specialized for specific cases of interest.
For instance, suppose $\vX$ is bounded and well-conditioned in the sense
that there exists $R < \infty$ such that $\Pr[\vX^\t \Sig^{-1} \vX \leq
R^2] = 1$, but $Y$ may still be heavy-tailed. 
Under this assumption we have the following result.

\begin{theorem}\label{thm:olsmain}
Assume $\Sig$ is not singular. There exists an algorithm that accepts as input a sample of size $n$ and a confidence parameter $\delta \in (0,1)$, and returns $\hat{\vw} \in \X$, such that with probability at least $1-\delta$,
for $n \geq O(R^2 \log(R) \log(e/\delta))$,
\begin{equation*}
L^\sq(\hat\vw)
\leq \biggl( 1 + O\biggl( \frac{R^2 \log(1/\delta)}{n} \biggr) \biggr)
L^\sq_\opt.
\end{equation*}
\end{theorem}
This theorem is stated in full as \thmref{ols} in \secref{regression}.
Note that 
\[
\E(\vX^\t \Sig^{-1} \vX) = \E\tr(\vX^\t \Sig^{-1}\vX) = \tr(\Id) = d,
\] 
so $R = \Omega(\sqrt{d})$.
$R^2$ is closely related to a \emph{condition number} for the distribution of $\vX$. For instance, if $\P[\norm{\vX}=1]=1$,
then $R^2 \leq d \frac{\lambda_{\max}(\Sig)}{\lambda_{\min}(\Sig)}$.
This result is minimax optimal up to logarithmic factors~\citep[see,
\emph{e.g.},][]{Nussbaum99}.
We also remark that the boundedness assumption can be replaced by a
subgaussian assumption on $\vX$, in which case the sample size requirement
becomes $O(d \log(1/\delta))$. We give analogous guarantees for the case of regularized least squares in a possibly finite-dimensional Hilbert space in \thmref{ridge}, \secref{regression}.

\subsection{Other Applications, Comparisons, and Extensions}

The general method studied here allows handling heavy tails in other applications as well. We give two examples in \secref{other}. First, we consider parameter estimation using $L^1$-regularized linear least squares regression (Lasso) under random subgaussian design. We show that using the above approach, parameter estimation bounds can be guaranteed for general bounded variance noise, including heavy-tailed noise. This contrasts with standard results that assume sub-Gaussian noise. 
Second, we show that low-rank covariance matrix approximation can be
obtained for heavy-tailed distributions, under a bounded $4+\epsilon$
moment assumption. These two applications have been analyzed also in
the independent and simultaneous work of \citet{Minsker13}.

All the results above are provided using a specific solution to the
Robust Distance Approximation problem, which is easy to implement for
any metric space. For the case of a fully known metric, in a Banach or
a Hilbert space, \citet{Minsker13} proposed a different solution, which is based on the geometric median. In \secref{geometric}, we provide a detailed comparison of the approximation factor achieved by each approach, as well as some general lower bounds. Several interesting open questions remain regarding this general problem.

Lastly, in \secref{output}, we give a short proof to the intuitive fact that in some prediction problem, one can replace Robust Distance Approximation with taking the median of the predictions of the input estimators. This gives a possible improper-learning algorithm for relevant learning settings.

All of the techniques we have developed in this work are simple enough
to implement and empirically evaluate, and indeed in some simulated
experiments, we have verified the improvements over standard methods
such as the empirical mean when the data follow heavy-tailed
distributions.
However, at present, the relatively large constant factors in our
bounds are real enough to restrict the empirical improvements only to
settings where very high confidence (\emph{i.e.}, small values of
$\delta$) is required.
By contrast, with an appropriately determined noise variance, the
techniques of~\cite{catoni} and~\cite{lugosi} may yield improvements
more readily.
Nevertheless, since our techniques are more general in some respects,
it is worth investigating whether they can be made more practical
(\emph{e.g.}, with greater sample reuse or overlapping groups), and we
plan to do this in future work.

\section{The Core Techniques}\label{sec:core}

In this section we present the core technique used for achieving exponential concentration.
We first demonstrate the underlying principle via the median-of-means
estimator, and then explain the generalization to arbitrary metric spaces. Finally, we show a new generalization that supports noisy feature measurements.

\subsection{Warm-up: Median-of-Means Estimator}
\label{sec:median-of-means}

We first motivate the estimation procedure by
considering the special case of estimating a scalar population mean using a
\emph{median-of-means} estimator, given in \algref{median-of-means}.
This estimator, heavily used in the streaming algorithm
literature~\citep{alon99} (though a similar technique also appears in~\cite{NY} as noted in \cite{Levin-notes}), partitions a
sample into $k$ equal-size groups, and returns the median of the sample
means of each group.
Note that the possible non-uniqueness of the median does not affect
the result; the arguments below apply to any one of them.
The input parameter $k$ should be thought of as a constant determined by the
desired confidence level (\emph{i.e.}, $k = \Theta(\log(1/\delta))$ for confidence
$\delta \in (0,1)$). It is well known that the median-of-means achieves estimation with exponential concentration. The following proposition gives a simple statement and proof. The constant $6$ in the statement (see \eqref{medbound} below) is lower the constant in the analysis of \citet[Proposition 1]{LO11}, which is $2\sqrt{6 e} \approx 8.08$, but we require a larger value of $n$. By requiring an even larger $n$, the constant in the statement below can approach $3\sqrt{3}$.

\begin{algorithm}[t]
\caption{Median-of-means estimator}
\label{alg:median-of-means}
\begin{algorithmic}[1]
\renewcommand{\algorithmicrequire}{\textbf{input}}
\renewcommand{\algorithmicensure}{\textbf{output}}
\REQUIRE Sample $S \subset \reals$ of size $n$, number of groups $k \in \nats$ such that $k \leq n/4$.

\ENSURE Population mean estimate $\hat\mu \in \reals$.

\STATE Randomly partition $S$ into $k$ subsets $S_1, S_2, \dotsc, S_k$, each
of size at least $\floor{n/k}$.

\STATE For each $i \in [k]$, let $\mu_i \in \reals$ be the sample mean of
$S_i$.

\STATE Return $\hat\mu := \median\{\mu_1,\mu_2,\dotsc,\mu_k\}$.

\end{algorithmic}
\end{algorithm}

\begin{proposition}
\label{prop:median-of-means}
Let $x$ be a random variable with mean $\mu$ and variance $\sigma^2 <
\infty$, and let $S$ be a set of $n$ independent copies of $x$.
Assume $k \leq n/2$.
With probability at least $1-e^{-k/4.5}$, the estimate $\hat\mu$ returned
by \algref{median-of-means} on input $(S,k)$ satisfies $|\hat\mu - \mu|
\leq \sigma\sqrt{8k/n}$.
Therefore, if $k = 4.5\ceil{\log(1/\delta)}$ and $n \geq 18 \ceil{\log(1/\delta)}$, then with probability at least $1-\delta$, 
\begin{equation}\label{eq:medbound}
  |\hat\mu - \mu|
  \leq 6\sigma\sqrt{\frac{\ceil{\log(1/\delta)}}{n}}.
\end{equation}
\end{proposition}
\begin{proof}
First, assume $k$ divides $n$.
Pick any $i \in [k]$, and observe that $S_i$ is an i.i.d.~sample of size
$n/k$.
Therefore, by Chebyshev's inequality, $\Pr[ |\mu_i - \mu| \leq
\sqrt{6\sigma^2k/n} ] \geq 5/6$.
For each $i \in [k]$, let $b_i := \ind{|\mu_i - \mu| \leq
\sqrt{6\sigma^2k/n}}$.
Note that the $b_i$ are independent indicator random variables, each with
$\E(b_i) \geq 5/6$.
By Hoeffding's inequality, $\Pr[ \sum_{i=1}^k b_i > k/2 ] \geq
1-e^{-k/4.5}$.
In the event that $\sum_{i=1}^k b_i > k/2$, at least half of the $\mu_i$
are within $\sqrt{6\sigma^2k/n}$ of $\mu$, which means that the same holds for the median of the
$\mu_i$. If $k$ does not divide $n$ then the analysis can be carried out by substituting $n$ with $\floor{n/k}k \geq n-k \geq \frac{3}{4}n$, which scales the guarantee by a factor of $\sqrt{4/3}$.
\end{proof}

Using the terminology of Robust Distance Approximation with the metric $\rho(x,y) = |x-y|$, the proof shows that with high probability over the choice of $W$, $\omed_W(\mu,0) \leq \sqrt{12\sigma^2k/n}$. The result then immediately follows because on the space $(\reals,\rho)$, the median is a Robust Distance Approximation procedure with $C_0 = 1$.

\begin{remark}[\citeauthor{catoni}'s M-estimator]
\citet{catoni} proposes a mean estimator $\hat\mu$ that satisfies
$|\hat\mu-\mu| = O(\sigma\sqrt{\log(1/\delta)/n})$ with probability at
least $1-\delta$.
Remarkably, the leading constant in the bound is asymptotically
optimal: it approaches $\sqrt{2}$ as $n\to\infty$.
However, the estimator takes both $\delta$ and $\sigma$ as inputs.
\citeauthor{catoni} also presents an estimator that takes only
$\sigma$ as an input; this estimator guarantees a $O(\sigma
\log(1/\delta) / \sqrt{n})$ bound for all values of $\delta >
\exp(1-n/2)$ \emph{simultaneously}.
\end{remark}

\begin{remark}[Empirical mean]
\label{remark:empirical-mean}
\citet{catoni} shows that the empirical mean cannot provide a
qualitatively similar guarantee.
Specifically, for any $\sigma > 0$ and $\delta \in (0,1/(2e))$, there is
a distribution with mean zero and variance $\sigma^2$ such that the
empirical average $\hat\mu_{\emp}$ of $n$ i.i.d.~draws satisfies
\begin{equation}
\label{eq:lb}
\Pr\biggl[
|\hat\mu_{\emp}| \geq
\frac{\sigma}{\sqrt{2n\delta}}
\Bigl( 1 - \frac{2e\delta}{n} \Bigr)^{\frac{n-1}{2}}
\biggr] \geq
2\delta .
\end{equation}
Therefore the deviation of the empirical mean necessarily scales with
$1/\sqrt{\delta}$ rather than $\sqrt{\log(1/\delta)}$ (with probability
$\Omega(\delta)$).
\end{remark}

\subsection{Generalization to Arbitrary Metric Spaces}
We now consider a simple generalization of the median-of-means estimator for
arbitrary metric spaces, first mentioned in \cite{NY}.
Let $\X$ be the parameter (solution) space, $\vw_\opt \in \X$ be a
distinguished point in $\X$ (the target solution), and $\rho$ a metric on
$\X$ (in fact, a pseudometric suffices).

The first abstraction captures the generation of candidate solutions
obtained from independent subsamples.
We assume there is an oracle $\APPROX_{\rho,\veps}$ which satisfies the following assumptions.
\begin{assumption}\label{as:approx1}
A query to $\APPROX_{\rho,\veps}$ returns a random $\vw \in \X$ such that
\begin{equation*}
\Pr\Bigl[ \rho(\vw_\opt,\vw) \leq \veps \Bigr] \geq 2/3 . 
\end{equation*}
\end{assumption}
Note that the $2/3$ could be replaced by another constant larger than half; we have not optimized the constants. The second assumption regards statistical independence.
For an integer $k$, let $\vw_1,\ldots,\vw_k$ be responses to $k$ separate queries to $\APPROX_{\rho,\veps}$. 
\begin{assumption}\label{as:approx2}
$\vw_1,\ldots,\vw_k$ are statistically independent.
\end{assumption}

The proposed procedure, given in \algref{main}, generates $k$ candidate
solutions by querying $\APPROX_{\rho,\veps}$ $k$ times, and then selecting a single candidate using a generalization of the median.
Specifically, for each $i \in [k]$, the smallest ball centered at $\vw_i$
that contains more than half of $\{ \vw_1, \vw_2, \dotsc, \vw_k \}$ is
determined; the $\vw_i$ with the smallest such ball is returned.
If there are multiple such $\vw_i$ with the smallest radius ball, any
one of them may be selected.
This selection method is a Robust Distance Approximation procedure. The proof is given below and illustrated in \figref{proof}. \cite{NY} proposed a similar technique, however their formulation relies on knowledge of $\veps$.

\begin{proposition}\label{prop:genmed}
Let $r_i := \min\{ r \geq 0 : |\ball_\rho(\vw_i,r) \cap W| > k/2 \}$. Selecting $\vw_{i_\opt}$ such that $i_\opt = \argmin_i r_i$ is a Robust Distance Approximation procedure with $C_0 = 3$.
\end{proposition}
\begin{proof}
Assume that $\omed(w_\opt, 0) \leq \veps$. Then $|\ball_\rho(\vw_\opt,\veps) \cap W| > k/2$.
For any $\vv \in \ball_\rho(\vw_\opt,\veps) \cap W$, by the
triangle inequality, $|\ball_\rho(\vv,2\veps) \cap W| > k/2$.
This implies that $r_{i_\opt} \leq 2\veps$, and so $|\ball_\rho(\vw_{i_\opt},2\veps) \cap W| > k/2$.
By the pigeonhole principle, $\ball_\rho(\vw_\opt,\veps) \cap
\ball_\rho(\vw_{i_\opt},2\veps) \neq\emptyset$.
Therefore, by the triangle inequality again, $\rho(\vw_\opt,\vw_{i_\opt}) \leq
3\veps$.
\end{proof}

\begin{figure}[t]
\begin{center}
\begin{tikzpicture}[scale = 0.1]
\node at (0,-3) {$\vw_\opt$};
\draw (0,0) circle (1);
\draw (0,0) circle (10);
\draw[->] (0,0) --  node[above] {$\varepsilon$} (-7,7);
 \fill (10,0) circle (1);
 \fill (0,10) circle (1);
 \fill (0,-10) circle (1);
 \draw (17,0) circle (12);
 \node at (17,-3) {$\hat{\vw}$};
 \draw[->] (17,0) -- node[above] {$r_{i_\opt}$} (29,0);
 \fill (17,0) circle (1);
 \fill (-10,0) circle (1);
 \fill (40,-10) circle (1);
 \fill (20,10) circle (1);
 \fill (25,-5) circle (1);
 \fill (7,-5) circle (1);
 \fill (8,4) circle (1);
 \fill (60, 10) circle (1);
 \end{tikzpicture}
 \caption{The argument in the proof of \propref{genmed}, illustrated on the Euclidean plane. If more than half of the $\vw_i$ (depicted by full circles) are within $\veps$ of $\vw_\opt$ (the empty circle), then the selected $\vw_{i_\opt}$ is within $\veps + r_{i_\opt} \leq 3\veps$ of $\vw_\opt$.}
\label{fig:proof}
 \end{center}
\end{figure}
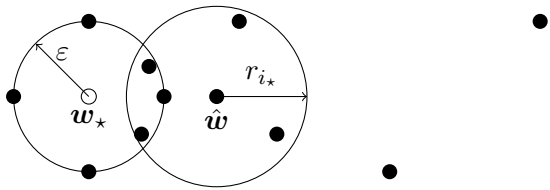

Again, the number of candidates $k$ determines the resulting confidence
level. The following theorem provides a guarantee for \algref{main}. We note that the resulting constants here might not be optimal in specific applications, since they depend on the arbitrary constant in \asref{approx1}.

\begin{algorithm}[t]
\caption{Robust approximation}
\label{alg:main}
\begin{algorithmic}[1]
\renewcommand{\algorithmicrequire}{\textbf{input}}
\renewcommand{\algorithmicensure}{\textbf{output}}
\REQUIRE Number of candidates $k$, query access to $\APPROX_{\rho,\veps}$.

\ENSURE Approximate solution $\hat\vw \in \X$.

\STATE Query $\APPROX_{\rho,\veps}$ $k$ times. Let $\vw_1,\ldots,\vw_k$ be the responses to the queries; set $W := \{\vw_1, \vw_2, \dotsc, \vw_k\}$.

\STATE For each $i \in [k]$, let $r_i := \min\{ r \geq 0 :
|\ball_\rho(\vw_i,r) \cap W| > k/2 \}$; set $i_\opt := \arg\min_{i \in [k]}
r_i$.

\STATE Return $\hat\vw := \vw_{i_\opt}$.

\end{algorithmic}
\end{algorithm}

\begin{proposition}\label{prop:alg}
Suppose that \asref{approx1} and \asref{approx2} hold. Then, with probability at least $1-e^{-k/18}$, \algref{main} returns $\hat\vw \in
\X$ satisfying $\rho(\vw_\opt,\hat{\vw}) \leq 3\veps$.
\end{proposition}
\begin{proof}
For each $i \in [k]$, let $b_i := \ind{\rho(\vw_\opt,\vw_i) \leq \veps}$.
Note that the $b_i$ are independent indicator random variables, each with
$\E(b_i) \geq 2/3$.
By Hoeffding's inequality, $\Pr[ \sum_{i=1}^k b_i > k/2 ] \geq
1-e^{-k/18}$.
In the event that $\sum_{i=1}^k b_i > k/2$, more than half of the $\vw_i$
are contained in the ball of radius $\veps$ around $\vw_\opt$, that is $\omed_W(w_\opt,0) \leq \veps$. The result follows from \propref{genmed}.
\end{proof}

\subsection{Random Distance Measurements}\label{sec:randomdist}

In some problems, the most appropriate metric on $\X$ in which to measure
accuracy is not directly computable.
For instance, the metric may depend on population quantities which can only
be estimated; moreover, the estimates may only be relatively accurate with
some constant probability. For instance, this is the case when the metric depends on the population covariance matrix, a situation we consider in \secref{heavyx}.

To capture such cases, we assume access to a metric estimation oracle as follows. Let $\vw_1,\ldots,\vw_k$ be responses to $k$ queries to $\APPROX_{\rho,\epsilon}$. The metric estimation oracle, denoted $\DIST^j_\rho$, provides (possibly via a random process) a function $f_j:\X \rightarrow \reals_+$.  $f_j(\vv)$ will be used as an estimate of $\rho(\vv,\vw_j)$. This estimate is required to be weakly accurate, as captured by the following definition of the random variables $Z_j$.
Let $f_1,\ldots,f_k$ be responses to queries to $\DIST_\rho^1,\ldots,\DIST_\rho^k$, respectively. For $j \in [k]$, define
\[
Z_j := \ind{\forall \vv \in \X, \quad(1/2) \rho(\vv,\vw_j) \leq f_j(\vv) \leq 2\rho(\vv,\vw_j)}.
\]
$Z_j = 1$ indicates that $f_j$ provides a sufficiently accurate estimate of the distances from $\vw_j$. Note that $f_j$ need not correspond to a metric.
We assume the following.

\begin{assumption}\label{as:dist1}
For any $j\in [k]$, $\Pr[Z_j = 1] \geq 8/9$.
\end{assumption}
We further require the following independence assumption.
\begin{assumption}\label{as:dist2}
The random variables $Z_1,\ldots,Z_k$ are statistically independent.
\end{assumption}
Note that there is no assumption on the statistical relationship between $Z_1,\ldots,Z_k$ and $\vw_1,\ldots,\vw_k$.

\algref{main2} is a variant of \algref{main} that simply replaces
computation of $\rho$-distances with computations using the functions returned by querying the $\DIST^j_\rho$'s. The resulting selection procedure is, with high probability, a Robust Distance Approximation.

\begin{lemma}\label{lem:distmed}
Consider a run of \algref{main2}, with output $\hat{\vw}$. Let $Z_1,\ldots,Z_k$ as defined above, and suppose that \asref{dist1} and \asref{dist2} hold. Then, with probability at least $1-e^{-k/648}$, 
\[
\rho(\hat{\vw},\vw_\opt) \leq 9\cdot\omed_W(\vw_\opt,\frac{5}{36}),
\]
where $W = \{\vw_1,\ldots,\vw_k\}$.
\end{lemma}
\begin{proof}
By Assumptions \ref{as:dist1} and \ref{as:dist2}, and by Hoeffding's inequality,
\begin{equation}\label{eq:Z}
  \Pr\left[ \sum_{j=1}^k Z_j > \frac{31}{36}k \right] \geq 1-e^{-k/648}
\end{equation}
Assume this event holds, and denote $\veps = \omed_W(\vw_\opt,\frac{5}{36})$. We have $|B(\vw_\opt, \veps) \cap W| \geq \frac{23}{36}k$.

Let $i \in [k]$ such that $\vw_i \in \ball_\rho(\vw_\opt, \veps)$.
Then, for any $j \in [k]$ such that $\vw_j \in \ball_\rho(\vw_\opt, \veps)$, by the triangle inequality $\rho(\vw_i,\vw_j) \leq 2\veps$.
There are at least $\frac{23}{36}k$ such indices $j$, therefore for more than $k/2$ of the indices $j$, we have 
\[
\rho(\vw_i,\vw_j) \leq 2\veps \text{ and } Z_j = 1.
\]
For $j$ such that this holds, by the definition of $Z_j$, $f_j(\vw_i) \leq 4\veps.$ It follows that  $r_i := \median\{f_j(\vw_i) \mid j \in [k]\} \leq 4\epsilon$.

Now, let $i \in [k]$ such that $\vw_i \notin B(\vw_\opt, 9\epsilon)$. 
Then, for any $j \in [k]$ such that $\vw_j \in \ball_\rho(\vw_\opt, \veps)$,
by the triangle inequality $\rho(\vw_i,\vw_j) \geq \rho(\vw_\opt, \vw_i) - \rho(\vw_\opt, \vw_j) > 8\veps$. As above, for more than $k/2$ of the indices $j$,
\[
\rho(\vw_i,\vw_j) > 8\veps \text{ and } Z_j = 1.
\]
For $j$ such that this holds, by the definition of $Z_j$, $f_j(\vw_i) > 4\veps.$ It follows that $r_i := \median\{f_j(\vw_i) \mid j \in [k]\} > 4\epsilon$.

By \eqref{Z}, We conclude that with probability at least $1-\exp(-k/648)$,
\begin{enumerate}
\item $r_i \leq 4\veps$ for all $\vw_i \in W \cap
\ball_\rho(\vw_\opt,\veps)$, and
\item $r_i > 4\veps$ for all $\vw_i \in W \setminus
\ball_\rho(\vw_\opt,9\veps)$.
\end{enumerate}
In this event the $\vw_i \in W$ with the smallest $r_i$ satisfies $\vw_i
\in \ball_\rho(\vw_\opt,9\veps)$.
\end{proof}

\begin{algorithm}[t]
\caption{Robust approximation with random distances}
\label{alg:main2}
\begin{algorithmic}[1]
\renewcommand{\algorithmicrequire}{\textbf{input}}
\renewcommand{\algorithmicensure}{\textbf{output}}
\REQUIRE Number of candidates $k$, query access to $\APPROX_{\rho,\veps}$,
query access to $\DIST_\rho$.

\ENSURE Approximate solution $\hat\vw \in \X$.

\STATE Query $\APPROX_{\rho,\veps}$ $k$ times. Let $\vw_1,\ldots,\vw_k$ be the responses to the queries; set $W := \{\vw_1, \vw_2, \dotsc, \vw_k\}$.

\STATE For $i \in [k]$, let $f_i$ be the response of $\DIST^j_\rho$, and set
\mbox{$r_i := \median\{ f_j(\vw_i)
: j \in [k] \}$}; set $i_\opt := \arg\min_{i \in [k]} r_i$.\label{step:select}

\STATE Return $\hat\vw := \vw_{i_\opt}$.

\end{algorithmic}
\end{algorithm}

The properties of the approximation procedure and of $\APPROX_{\rho,\epsilon}$ are combined to give a guarantee for \algref{main2}. 
\begin{theorem}\label{thm:alg2}
Suppose that Assumptions  \ref{as:approx1},\ref{as:approx2},\ref{as:dist1},\ref{as:dist2} all hold. 
With probability at least \mbox{$1-2e^{-k/648}$}, \algref{main2} returns
$\hat\vw \in \X$ satisfying $\rho(\vw_\opt,\hat{\vw}) \leq 9\veps$.
\end{theorem}
\begin{proof}
For each $i \in [k]$, let $b_i := \ind{\rho(\vw_\opt,\vw_i) \leq \veps}$.
By Assumptions \ref{as:approx1} and \ref{as:approx2}, the $b_i$ are independent indicator random variables, each with
$\E(b_i) \geq 2/3$. By Hoeffding's inequality, $\Pr[ \sum_{i=1}^k b_i > \frac{23}{36}k ] \geq 1 - e^{-k/648}$. The result follows from \lemref{distmed} and a union bound.
\end{proof}
In the following sections we show several applications of these general techniques.

\section{Minimizing Strongly Convex Losses}
\label{sec:approx}

In this section we apply the core techniques to the problem of
approximately minimizing strongly convex losses, which includes least
squares linear regression as a special case.

\subsection{Preliminaries}
\label{sec:approx-prelim}

Suppose $(\X,\nrm)$ is a Banach space, with the metric $\rho$ induced by the norm $\nrm$. We sometimes denote the metric by $\nrm$ as well.
Denote by $\dualnrm$ the dual norm, so $\dualnorm{\vy} =
\sup\{\dotprod{\vy,\vx} \colon \vx \in \X, \norm{\vx} \leq 1 \}$ for $\vy
\in \X^*$.

The derivative of a differentiable function $f \colon \X \to \reals$ at
$\vx \in \X$ in direction $\vu \in \X$ is denoted by $\dotprod{\nabla
f(\vx),\vu}$.
We say $f$ is \emph{$\alpha$-strongly convex with respect to $\nrm$}
if 
\[
f(\vx) \geq f(\vx') + \dotprod{\grad f(\vx'), \vx - \vx'} +
\frac\alpha2 \norm{\vx - \vx'}^2
\]
for all $\vx, \vx' \in \X$; it is \emph{$\beta$-smooth with respect to
$\nrm$} if for all $\vx, \vx' \in \X$
\[
f(\vx) \leq f(\vx') + \dotprod{\grad f(\vx'), \vx - \vx'} + \frac\beta2
\norm{\vx - \vx'}^2.
\]
We say $\nrm$ is \emph{$\gamma$-smooth} if $\vx \mapsto \frac12
\norm{\vx}^2$ is $\gamma$-smooth with respect to $\nrm$.
We define
$n_{\alpha}$ to be the smallest sample size such that the following
holds:
With probability $\geq 5/6$ over the choice of an i.i.d.~sample $T$ of size
$|T| \geq n_{\alpha}$ from $\cD$, for all $\vw \in \X$,
\begin{equation}
\label{eq:empirical-sc}
L_T(\vw) \geq L_T(\vw_\opt) + \dotprod{\grad L_T(\vw_\opt), \vw - \vw_\opt} +
\frac{\alpha}2 \norm{\vw - \vw_\opt}^2.
\end{equation}
In other words, the sample $T$ induces a loss $L_T$ which is
$\alpha$-strongly convex around $\vw_\opt$.\footnote{Technically, we only need
the sample size to guarantee \eqref{empirical-sc} for all $\vw \in
\ball_\nrm(\vw_\opt,r)$ for some $r>0$.}
We assume that $n_{\alpha} < \infty$ for some $\alpha > 0$. 

We use the following facts in our analysis.
\begin{proposition}[\citealp{smooth-loss}]
\label{prop:smooth-loss}
If a non-negative function $f \colon \X \to \reals_+$ is $\beta$-smooth
with respect to $\nrm$, then $\dualnorm{\grad f(\vx)}^2 \leq 4\beta
f(\vx)$ for all $\vx \in \X$.
\end{proposition}

\begin{proposition}[\citealp{JN08}]
\label{prop:smooth-norm}
Let $\vX_1, \vX_2, \dotsc, \vX_n$ be independent copies of a zero-mean
random vector $\vX$, and let $\nrm$ be $\gamma$-smooth.
Then $\E \norm{n^{-1} \sum_{i=1}^n \vX_i}^2 \leq (\gamma/n) \E \norm{\vX}^2$.
\end{proposition}

Recall that $\cZ$ is a data space, and $\cD$ is a distribution over $\cZ$. Let $Z$ be a
$\cZ$-valued random variable with distribution $\cD$. Let $\ell \colon \cZ \times \X \to \reals_+$ be a non-negative loss
function, and for $\vw \in \X$, let $L(\vw) := \E(\ell(Z,\vw))$ be the expected loss.
Also define the empirical loss with respect to a sample $T$ from $\cZ$, 
$L_T(\vw) := |T|^{-1} \sum_{z \in T}
\ell(z,\vw)$.
To simplify the discussion throughout, we assume $\ell$ is differentiable,
which is anyway our primary case of interest.
We assume that $L$ has a unique minimizer $\vw_\opt := \arg\min_{\vw \in \X}
L(\vw)$.%
\footnote{This holds, for instance, if $L$ is strongly convex.}
Let $L_\opt := \min_{\vw} L(\vw)$.
Set $\vw_\opt$ such that $L_\opt = L(\vw_\opt)$.%

\subsection{Subsampled Empirical Loss Minimization}

To use \algref{main}, we implement $\APPROX_{\nrm,\veps}$ based on loss minimization over subsamples, as follows:
Given a sample $S \subseteq \cZ$, randomly partition $S$ into $k$
groups $S_1, S_2, \dotsc, S_k$, each of size at least $\floor{|S|/k}$, and let the response to the
$i$-th query to $\APPROX_{\nrm,\veps}$ be the loss minimizer on $S_i$, \emph{i.e.}, $\vw_i = \arg\min_{\vw\in \X} L_{S_i}(\vw)$. We call this implementation \emph{subsampled empirical loss minimization}. Clearly, if $S$ is an i.i.d.\ sample from $\cD$, then $\vw_1,\ldots,\vw_k$ are statistically independent, and so \asref{approx2} holds. Thus, to apply \propref{alg}, it is left to show that \asref{approx1} holds as well.\footnote{An approach akin to the bootstrap technique \citep{Efron79} could also seem natural here: In this approach, $S_1,\ldots,S_k$ would be generated by randomly sub-sampling from $S$, with possible overlap between the sub-samples. However, this approach does not satisfy \asref{approx2}, since loss minimizers of overlapping samples are not statistically independent.}

The following lemma proves that \asref{approx1} holds under these assumptions with 
\begin{equation}
\label{eq:loss-eps}
\veps
:=
\sqrt{\frac{32\gamma k \E\dualnorm{\grad\ell(Z,\vw_\opt)}^2} {n\alpha^2}}
.
\end{equation}

\begin{lemma}
\label{lem:chebyshev}
Let $\veps$ be as defined in \eqref{loss-eps}. Assume $k \leq n/4$, and that $S$ is an i.i.d.~sample from $\cD$ of size
$n$ such that $\floor{n/k} \geq n_\alpha$.
Then subsampled empirical loss minimization using the sample $S$ is a
correct implementation of $\APPROX_{\nrm,\veps}$ for up to $k$ queries.
\end{lemma}
\begin{proof}
Let $T = \floor{n/k}$.
Since $n \geq 4k$, we have $\floor{n/k}k \geq n-k \geq \frac{3}{4}n$, therefore 
$1/T \leq \frac{4k}{3n}$.
It is clear that $\vw_1, \vw_2, \dotsc, \vw_k$ are independent by the
assumption.
Fix some $i \in [k]$.
Observe that $\grad L(\vw_\opt) = \E(\grad\ell(Z,\vw_\opt)) = 0$, and therefore
 by \propref{smooth-norm}:
\[
\E\dualnorm{\grad L_{S_i}(\vw_\opt)}^2 \leq (\gamma/T)\E\dualnorm{\grad\ell(Z,\vw_\opt)}^2 \leq \frac{4\gamma k}{3n}\E\dualnorm{\grad\ell(Z,\vw_\opt)}^2.
\]
By Markov's inequality, 
\[
\Pr\biggl[\dualnorm{\grad L_{S_i}(\vw_\opt)}^2 \leq \frac{8\gamma k}{n}
\E(\dualnorm{\grad\ell(Z,\vw_\opt)}^2) \biggr] \geq \frac56 .
\]

Moreover, the assumption that $\floor{n/k} \geq n_{\alpha}$ implies that with
probability at least $5/6$, \eqref{empirical-sc} holds for $T = S_i$.
By a union bound, both of these events hold simultaneously with probability
at least $2/3$.
In the intersection of these events, letting $\vw_i := \arg\min_{\vw \in
\X} L_{S_i}(\vw)$,
\begin{align*}
(\alpha/2) \norm{\vw_i - \vw_\opt}^2
& \leq - \dotprod{\grad L_{S_i}(\vw_\opt), \vw_i - \vw_\opt} + L_{S_i}(\vw_i) -
L_{S_i}(\vw_\opt) \\
& \leq \dualnorm{\grad L_{S_i}(\vw_\opt)} \norm{\vw_i - \vw_\opt},
\end{align*}
where the last inequality follows from the definition of the dual norm, and
the optimality of $\vw_i$ on $L_{S_i}$.
Rearranging and combining with the above probability inequality implies
\[
\Pr\Bigl[ \norm{\vw_i - \vw_\opt} \leq \veps \Bigr] \geq \frac23
\]
as required.
\end{proof}

Combining \lemref{chebyshev} and \propref{alg} gives the following theorem.
\begin{theorem}
\label{thm:loss}
Let $n_{\alpha}$ be as defined in \secref{approx-prelim}, and assume
that $\nrm_*$ is $\gamma$-smooth.
Also, assume $k := 18\lceil\log(1/\delta)\rceil$, $n \geq 72 \ceil{\log(1/\delta)}$, and that $S$ is an
i.i.d.~sample from $\cD$ of size $n$ such that $\floor{n/k}\geq
n_{\alpha}$.
Finally, assume \algref{main2} uses the subsampled empirical loss
minimization to implement $\APPROX_{\nrm,\veps}$, where $\veps$ is as in
\eqref{loss-eps}.
Then with probability at least $1-\delta$, the parameter $\hat\vw$ returned
by \algref{main} satisfies
\begin{equation*}
\norm{\hat\vw - \vw_\opt}
\leq 72\sqrt{\frac{\gamma \lceil\log(1/\delta)\rceil
\E\dualnorm{\grad\ell(Z,\vw_\opt)}^2} {n\alpha^2}}
.
\end{equation*}
\end{theorem}

We give an easy corollary of \thmref{loss} for the case where $\ell$ is
smooth. This is the full version of \thmref{lossmain}.
\begin{corollary}
\label{cor:loss}
Assume the same conditions as~\thmref{loss}, and also that:
\begin{itemize}
\item $\vw \mapsto \ell(z,\vw)$ is $\beta$-smooth with respect to
$\nrm$ for all $z \in \cZ$;

\item $\vw \mapsto L(\vw)$ is $\bbeta$-smooth with respect to
$\nrm$.

\end{itemize}
Then with probability at least $1-\delta$,
\begin{equation*}
L(\hat\vw) \leq \biggl( 1 +
\frac{10368\beta\bbeta\gamma\lceil\log(1/\delta)\rceil}
{n\alpha^2}
\biggr) L(\vw_\opt) .
\end{equation*}
\end{corollary}
\begin{proof}
This follows from \thmref{loss} by first concluding that
$\E[\dualnorm{\grad\ell(Z,\vw_\opt)}^2]\leq 4\beta L(\vw_\opt)$, using the
$\beta$-strong smoothness assumption on $\ell$ and \propref{smooth-loss},
and then noting that $L(\hat\vw) - L(\vw_\opt) \leq
\frac{\bbeta}{2}\norm{\hat\vw - \vw_\opt}^2$, due to the strong smoothness of
$L$ and the optimality of $L(\vw_\opt)$.
\end{proof}

\corref{loss} implies that
for smooth losses, \algref{main} provides a
constant factor approximation to the optimal loss with a sample size
$\max\{n_{\alpha}, \gamma\beta\bbeta/\alpha^2\} \cdot O(\log(1/\delta))$
(with probability at least $1-\delta$).
In subsequent sections, we exemplify cases where the two arguments of the
$\max$ are roughly of the same order, and thus imply a sample size
requirement of $O(\gamma\bbeta\beta/\alpha^2 \log(1/\delta))$.
Note that there is no dependence on the optimal loss $L(\vw_\opt)$ in the
sample size, and the algorithm has no parameters besides $k =
O(\log(1/\delta))$.

We can also obtain a variant of \thmref{loss} based on \algref{main2} and \thmref{alg2},
in which we assume that there exists some sample size $n_{k,\DIST_\nrm}$
that allows $\DIST_\nrm$ to be correctly implemented using an i.i.d.~sample of size at least $n_{k,\DIST_\nrm}$.
Under such an assumption, essentially the same guarantee as in
\thmref{loss} can be afforded to \algref{main2} using the subsampled
empirical loss minimization to implement $\APPROX_{\nrm,\veps}$ (for
$\veps$ as in \eqref{loss-eps}) and the assumed implementation of
$\DIST_\nrm$. Note that since \thmref{alg2} does not require $\APPROX_{\nrm,\veps}$ and $\DIST_\nrm$ to be statistically independent, both can be implemented using the same sample.

\begin{theorem}
\label{thm:loss2}
Let $n_{\alpha}$ be as defined in \secref{approx-prelim},
$n_{k,\DIST_\nrm}$ be as defined above, and assume that $\nrm_*$ is
$\gamma$-smooth.
Also, assume $k := 648\lceil\log(2/\delta)\rceil$, $S$ is an
i.i.d.~sample from $\cD$ of size $n$ such that $n \geq \max\{4k,
n_{k,\DIST_\nrm}\}$, and $\floor{n/k} \geq n_\alpha$.
Further, assume \algref{main2} implements $\APPROX_{\nrm,\veps}$ using $S$ with subsampled empirical loss minimization, where $\veps$ is as in
\eqref{loss-eps}, and implements $\DIST_\nrm$ using $S$ as well.
Then with probability at least $1-\delta$, the parameter $\hat\vw$ returned
by \algref{main2} satisfies
\begin{equation*}
\norm{\hat\vw - \vw_\opt}
\leq 1296\sqrt{\frac{\gamma \lceil\log(2/\delta)\rceil
\E\dualnorm{\grad\ell(Z,\vw_\opt)}^2} {n\alpha^2}}
.
\end{equation*}
\end{theorem}

\begin{remark}[Mean estimation and empirical risk minimization]
\label{remark:erm}
The problem of estimating a scalar population mean is a special case of the
loss minimization problem, where $\cZ = \X = \reals$, and the loss function
of interest is the square loss $\ell(z,w) = (z-w)^2$.
The minimum population loss in this setting is the variance $\sigma^2$
of $Z$, \emph{i.e.}, $L(w_\opt) = \sigma^2$.
Moreover, in this setting, we have $\alpha = \beta = \bbeta = 2$, so the
estimate $\hat{w}$ returned by \algref{main} satisfies, with probability at
least $1-\delta$,
\[
L(\hat{w})
= \biggl( 1 + O\Bigl( \frac{\log(1/\delta)}{n} \Bigr) \biggr) L(w_\opt) .
\]

In Remark~\ref{remark:empirical-mean} a result from~\cite{catoni} is quoted which
implies that if $n = o(1/\delta)$, then the empirical mean $\hat{w}_{\emp}
:= \arg\min_{w \in \reals} L_S(w) = |S|^{-1} \sum_{z \in S} z$
(\emph{i.e.}, empirical risk (loss) minimization for this problem) incurs
loss
\[
L(\hat{w}_{\emp})
= \sigma^2 + (\hat{w}_{\emp} - w_\opt)^2
= (1 + \omega(1)) L(w_\opt)
\]
with probability at least $2\delta$.
Therefore empirical risk minimization cannot provide a qualitatively
similar guarantee as \corref{loss}.
It is easy to check that minimizing a regularized objective also does not
work, since any non-trivial regularized objective necessarily provides an
estimator with a positive error for some distribution with zero variance. 

\end{remark}

In the next section we use the analysis for general smooth and convex losses to derive new algorithms and bounds for linear regression.

\section{Least Squares Linear Regression}\label{sec:regression}

In linear regression, the parameter space $\X$ is a Hilbert space with inner product $\dotp{\cdot,\cdot}$, and $\cZ := \X \times \reals$, where in the finite-dimensional case, $\X = \reals^d$ for some finite integer $d$. 
The loss here is the squared loss, denoted by $\ell = \ell^\sq$, and defined as
\[
\ell^\sq((\vx,y),\vw) := \frac12(\vx^\t \vw - y)^2.
\]
The regularized squared loss, for $\lambda \geq 0$, is denoted
\[
\ell^\lambda((\vx,y),\vw) := \frac12(\dotp{\vx,\vw} - y)^2 +
\frac12\lambda\dotp{\vw,\vw}.
\] 
Note that $\ell^0 = \ell^\sq$.
We analogously define
$L^\sq$,
$L^\sq_T$,
$L^\sq_\opt$,
$L^\lambda$, \emph{etc.} as the squared-loss equivalents of $L,L_T,L_\opt$. Finally, denote by $\Id$ the identity operator on $\X$.

The proposed algorithm for regression (\algref{regression}) is as
follows. Set $k = C\log(1/\delta)$, where $C$ is a universal constant. First, draw $k$ independent random samples i.i.d.~from
$\cD$, and perform linear regression with $\lambda$-regularization on
each sample separately to obtain $k$ linear regressors.  Then,
use the same $k$ samples to generate $k$ estimates of the covariance matrix of the marginal of $\cD$ on the data space.
Finally, use the estimated covariances to select a single regressor
from among the $k$ at hand. The slightly simpler variants of steps \ref{step:sigmacalc} and \ref{step:sigma} can be used in some cases, as detailed below. 

In \secref{regmain}, the full results for regression, mentioned in \secref{overview}, are listed in full detail, and compared to previous work. The proofs are provided in \secref{analysis}. 

\begin{algorithm}[bth]
\caption{Regression for heavy-tails}
\label{alg:regression}
\begin{algorithmic}[1]
\renewcommand{\algorithmicrequire}{\textbf{input}}
\renewcommand{\algorithmicensure}{\textbf{output}}
\REQUIRE $\lambda \geq 0$, sample size $n$, confidence $\delta \in (0,1)$.
\ENSURE Approximate predictor $\hat\vw \in \X$.
\STATE Set $k := \ceil{C\ln(1/\delta)}$.
\STATE Draw $k$ random i.i.d.~samples $S_1,\ldots,S_k$ from $D$, each of size $\floor{n/k}$.
\STATE For each $i \in [k]$, let $\vw_i \in \argmin_{\vw \in \X}L^\lambda_{S_i}(\vw)$.
\STATE For each $i \in [k]$, $\Sig_{S_i} \leftarrow \frac{1}{|S_i|}\sum_{(\vx,\cdot)\in S_i} \vx \vx^\t$.\\
{[}\textbf{Variant}: $S \leftarrow \cup_{i \in [k]} S_i$; $\Sig_{S} \leftarrow \frac{1}{|S|}\sum_{(\vx,\cdot)\in S} \vx \vx^\t${]}.\label{step:sigmacalc}
\STATE For each $i \in [k]$, let $r_i$ be the median of the values in \[
\{\dotprod{\vw_i-\vw_j, (\Sig_{S_j}+\lambda \Id)(\vw_i - \vw_j)} \mid j \in [k]\setminus\{i\}\}.
\]\label{step:sigma}
{[}\textbf{Variant}: Use $\Sig_S$ instead of $\Sig_{S_{j}}${]}.
\STATE Set $i_\opt := \arg\min_{i \in [k]} r_i$.
\STATE Return $\hat\vw := \vw_{i_\opt}$.
\end{algorithmic}
\end{algorithm}

\subsection{Results}\label{sec:regmain}

Let $\vX \in \X$ be a random vector drawn according to the marginal of $\cD$ on $\X$, and let $\Sig : \X \to \X$ be
the second-moment operator $\vec{a} \mapsto \E(\vX \dotp{\vX,\vec{a}})$. For a finite-dimensional $\X$, $\Sig$ is 
simply the (uncentered) covariance matrix $\E[\vX \vX^\t]$.
For a sample $T := \{ \vX_1, \vX_2, \dotsc, \vX_m \}$ of $m$ independent
copies of $\vX$, denote by $\Sig_T : \X \to \X$ the empirical second-moment
operator $\vec{a} \mapsto m^{-1} \sum_{i=1}^m\vX_i \dotp{\vX_i,\vec{a}}$.

Consider first the finite-dimensional case, where $\X = \reals^d$, and assume $\Sig$ is not singular. 
Let $\norm{\cdot}_2$ denote the Euclidean norm in $\reals^d$. In this case we obtain a
guarantee for ordinary least squares with $\lambda = 0$.
The guarantee holds whenever the empirical estimate of $\Sig$ is close to the true $\Sig$ 
\emph{in expectation}, a mild condition that requires only bounded low-order moments. 
For concreteness, we assume the following condition.\footnote{As shown by
  \citet{SV}, Condition~\ref{cond:strong} holds for various
  heavy-tailed distributions (\emph{e.g.}, when $\vX$ has a product
distribution with bounded $4+\eps$ moments for some $\eps > 0$).
Condition~\ref{cond:strong} may be easily substituted with other moment conditions,
yielding similar results, at least up to logarithmic factors.}

\begin{condition}[\citealt{SV}]
\label{cond:strong}
There exists $c, \eta > 0$ such that
\[
\Pr\Bigl[ \|\Pi\Sig^{-1/2}\vX\|_2^2 > t \Bigr] \ \leq \ ct^{-1-\eta} ,
\quad \text{for $t > c \cdot \rank(\Pi)$}
\]
for every orthogonal projection $\Pi$ in $\reals^d$.%
\end{condition}
Under this condition, we show the following guarantee for least squares regression.
\begin{theorem}\label{thm:heavyx}
Assume $\Sig$ is not singular.
If $\vX$ satisfies \condref{strong} with some fixed parameters $c > 0$ and $\eta>0$, then
if \algref{regression} is run with $n \geq O(d \log(1/\delta))$ and $\delta \in (0,1)$, with probability at least $1-\delta$,
\[
L^\sq(\hat\vw) \leq
L^\sq_\opt 
+ O\left( \frac{\E\norm{\Sig^{-1/2}\vX(\vX^\t\vw_\opt-Y)}_2^2 \log(1/\delta)} {n}\right) .
\]
\end{theorem}
Our loss bound is given in terms of the following population quantity
\begin{equation}
\label{eq:output-moment}
\E\norm{\Sig^{-1/2}\vX(\vX^\t\vw_\opt-Y)}_2^2
\end{equation}
which we assume is finite.
This assumption only requires bounded low-order moments of $\vX$ and $Y$
and is essentially the same as the conditions from~\cite{AC11} (see the
discussion following their Theorem~3.1).
Define the following finite fourth-moment conditions:
\begin{align*}
  \kappa_1 & :=
  \frac{\sqrt{\E\norm{\Sig^{-1/2}\vX}_2^4}}{\E\norm{\Sig^{-1/2}\vX}_2^2}
  =
  \frac{\sqrt{\E\norm{\Sig^{-1/2}\vX}_2^4}}{d}
  <
  \infty
  \quad \text{and}
  \\
  \kappa_2 & :=
  \frac{\sqrt{\E(\vX^\t\vw_\opt-Y)^4}}{\E(\vX^\t\vw_\opt-Y)^2}
  = \frac{\sqrt{\E(\vX^\t\vw_\opt-Y)^4}}{L^\sq_\opt}
  < \infty.
\end{align*}
Under these conditions, $\E\norm{\Sig^{-1/2}\vX(\vX^\t\vw_\opt-Y)}_2^2 \leq
\kappa_1 \kappa_2 d L^\sq_\opt$ (via Cauchy-Schwartz); if $\kappa_1$ and
$\kappa_2$ are constant, then we obtain
the bound
\[ L^\sq(\hat\vw) \leq \biggl( 1 + O\biggl( \frac{d \log(1/\delta)}n
\biggr) \biggr) L^\sq_\opt \]
with probability $\geq 1-\delta$.
In comparison, the recent work of~\citet{AC11} proposes an estimator for
linear regression based on optimization of a robust loss function which achieves essentially the same guarantee as
\thmref{heavyx} (with only mild differences in the moment conditions, see the
discussion following their Theorem~3.1).
However, that estimator depends on prior knowledge about the response
distribution, and removing this dependency using Lepski's adaptation
method~\citep{lepski} may result in a suboptimal convergence rate.  
It is also unclear whether that estimator can be computed efficiently.

Other analyses for linear least squares regression and ridge
regression by~\cite{smooth-loss} and~\cite{HKZ12} consider specifically the empirical minimizer of the squared
loss, and give sharp rates of convergence to $L^\sq_\opt$. However, both of these require either boundedness of the
loss or boundedness of the approximation error.
In~\cite{smooth-loss}, the specialization of the main result to square
loss includes additive terms of order $O(\sqrt{L(\vw_\opt) b \log(1/\delta) / n} + b\log(1/\delta) / n)$, where $b > 0$ is
assumed to bound the square loss of any predictions almost surely.
In~\cite{HKZ12}, the convergence rate includes an additive term
involving almost-sure bounds on the approximation
error/non-subgaussian noise (The remaining terms are comparable
to~\eqref{ols} for $\lambda = 0$, and~\eqref{ridge} for $\lambda > 0$,
up to logarithmic factors). The additional terms preclude
multiplicative approximations to $L(\vw_\opt)$ in cases where the loss
or approximation error is unbounded. In recent work, \cite{Mendelson14} proposes a more subtle `small-ball' criterion for analyzing the performance of the risk minimizer. However, as evident from the lower bound in Remark~\ref{remark:erm}, the empirical risk minimizer cannot obtain the same type of guarantees as our estimator.

The next result is for the case where there exists $R < \infty$ such that $\Pr[\vX^\t \Sig^{-1} \vX \leq
R^2] = 1$ (and, here, we do not assume Condition~\ref{cond:strong}). In contrast, $Y$ may still be heavy-tailed.
Then, the following result can be derived using \algref{regression}. Moreover, the simpler \textbf{variant} of \algref{regression} suffices here.

\begin{theorem}\label{thm:ols}
Assume $\Sig$ is not singular. Let $\hat{\vw}$ be the output of the variant of \algref{regression} with $\lambda=0$. With probability at least $1-\delta$, for $n \geq O(R^2 \log(R) \log(1/\delta))$,
\begin{equation*}
L^\sq(\hat\vw)
\leq \biggl( 1 + O\biggl( \frac{R^2 \log(2/\delta)}{n} \biggr) \biggr)
L^\sq_\opt.
\end{equation*}
\end{theorem}
Note that $\E(\vX^\t \Sig^{-1} \vX) = \E\tr(\vX^\t \Sig^{-1}\vX) =
\tr(\Id) = d$, therefore $R = \Omega(\sqrt{d})$.
If indeed $R = \Theta(\sqrt{d})$, then a total sample size of $O(d
\log(d)\log(1/\delta))$ suffices to guarantee a constant factor
approximation to the optimal loss.
This is minimax optimal up to logarithmic factors~\citep{Nussbaum99}.
We also remark that the boundedness assumption can be replaced by a
subgaussian assumption on $\vX$, in which case the sample size requirement
becomes $O(d \log(1/\delta))$.

In recent work of~\citet{MJ13}, an algorithm based on stochastic gradient descent obtains multiplicative approximations to $L_\opt$, for
general smooth and strongly convex losses $\ell$, with a sample complexity
scaling with $\log(1/\tilde{L})$.
Here, $\tilde{L}$ is an upper bound on $L_\opt$, which must be known by the
algorithm.
The specialization of \citeauthor{MJ13}'s main result
to square loss implies a sample complexity of $\tilde{O}( d R^8
\log (1 /(\delta L^\sq_\opt))$ if $L^\sq_\opt$ is known.
In comparison, \thmref{ols} shows that $\tilde{O}(R^2\log(1/\delta))$ suffice when using our estimator.
It would be interesting to understand whether the bound for the stochastic
gradient method of~\cite{MJ13} can be improved, and whether knowledge of
$L_\opt$ is actually necessary in the stochastic oracle model.
We note that the main result of~\cite{MJ13} can be more generally
applicable than \thmref{loss}, because \cite{MJ13} only assumes that the
population loss $L(\vw)$ is strongly convex, whereas \thmref{loss} requires
the empirical loss $L_T(\vw)$ to be strongly convex for large enough
samples $T$.
While our technique is especially simple for the squared loss, it may be more
challenging to implement well for other losses, because the local norm
around $\vw_\opt$ may be difficult to approximate with an observable norm.
We thus leave the extension to more general losses as future work.

Finally, we also consider the case where $\X$ is a general,
infinite-dimensional Hilbert space, $\lambda > 0$, the norm of $\vX$ is
bounded, and $Y$ again may be heavy-tailed.
\begin{theorem}\label{thm:ridge}
Let $V > 0$ such that $\Pr[\dotp{\vX,\vX} \leq V^2] = 1$. 
Let $\hat{\vw}$ be the output of the variant of \algref{regression} with $\lambda  >0$.
With probability at least $1-\delta$,
as soon as $n \geq O((V^2/\lambda)\log(V/\sqrt{\lambda})\log(2/\delta))$,
\begin{equation*}
L^\lambda(\hat\vw)
\leq
\biggl( 1 + O\biggl( \frac{(1+V^2/\lambda)\log(2/\delta)}{n} \biggr) \biggr)
L^\lambda_\opt.
\end{equation*}

If the optimal unregularized squared loss $L^\sq_\opt$ is achieved by $\bar\vw \in \X$ with
$\dotp{\bar{\vw},\bar{\vw}} \leq B^2$, the choice 
$\lambda = \Theta(\sqrt{L^\sq_\opt V^2\log(2/\delta) / (B^2n)})$
yields that if $n \geq \tilde{O}(B^2V^2\log(2/\delta)/L^\sq_\opt)$
then
\begin{align}\label{eq:ridge}
&L^\sq(\hat\vw)\leq L^\sq_\opt\kern-2pt + O\biggl( \sqrt{\frac{L^\sq_\opt B^2V^2 \log(1/\delta)}{n}}
+ \frac{(L^\sq_\opt + B^2V^2)\log(1/\delta)}{n} \biggr).
\end{align}
\end{theorem}

By this analysis, a constant factor approximation for $L^\sq_\opt$ is
achieved with a sample of size $\tilde{O}(B^2V^2 \log(1/\delta) /
L^\sq_\opt)$.  As in the finite-dimensional setting, this rate is
known to be optimal up to logarithmic factors~\citep{Nussbaum99}. It
 is interesting to observe that in the non-parametric case, our
 analysis, like previous analyses, does require knowledge of $L_\opt$
 if $\lambda$ is to be set correctly, as in \cite{MJ13}.

 \subsection{Analysis}\label{sec:analysis}

 We now show how the analysis of \secref{approx} can be applied to analyze \algref{regression}.
 For a sample $T \subseteq \cZ$, if $L_T$ is twice-differentiable (which is
the case for squared loss), by Taylor's theorem, for any $\vw \in \X$, there exist $t \in [0,1]$
and $\tilde\vw = t \vw_\opt + (1-t) \vw$ such that
\begin{equation*}
L_T(\vw) = L_T(\vw_\opt)
+ \dotp{\grad L_T(\vw_\opt), \vw - \vw_\opt}
+ \frac12 \dotp{\vw - \vw_\opt, \grad^2 L_T(\tilde\vw) (\vw - \vw_\opt)} ,
\end{equation*}
Therefore, to establish a bound on $n_{\alpha}$, it suffices to control
\begin{equation}
\label{eq:hessian-control}
\Pr\biggl[
\inf_{\vec\delta\in\X \setminus \{\vec{0}\}, \tilde\vw\in\reals^d}
\frac{\dotp{\vec\delta, \grad^2 L_T(\tilde\vw) \vec\delta}}
{\norm{\vec\delta}^2}
\geq \alpha \biggr]
\end{equation}
for an i.i.d.~sample $T$ from $\cD$. The following lemma allows doing just that.

\begin{lemma}[Specialization of Lemma~1 from \citealt{Oliveria10}]
\label{lem:matrix-chernoff}
Fix any $\lambda \geq 0$, and assume $\dotp{\vX,(\Sig + \lambda
\Id)^{-1} \vX} \leq r_\lambda^2$ almost surely. 
For any $\delta \in (0,1)$, if $m \geq 80r_\lambda^2\ln(4m^2/\delta)$, then
with probability at least $1-\delta$, for all $\vec{a} \in \X$,
\[
\frac12 \dotp{\vec{a}, (\Sig + \lambda \Id) \vec{a}}
\leq \dotp{\vec{a}, (\Sig_T + \lambda \Id) \vec{a}}
\leq 2 \dotp{\vec{a}, (\Sig + \lambda \Id) \vec{a}}.
\]
\end{lemma}
We use the boundedness assumption for sake of simplicity; it is possible to
remove the boundedness assumption, and the logarithmic dependence on the
cardinality of $T$, under different conditions on $\vX$ (\emph{e.g.},
assuming $\Sig^{-1/2}\vX$ has subgaussian projections, see \citealt{LPRTJ05}).
We now prove \thmref{ols}, \thmref{ridge} and \thmref{heavyx}.

\subsubsection{Ordinary Least Squares in Finite Dimensions}\label{sec:finite}

Consider first ordinary least squares in the finite-dimensional case. In
this case $\X = \reals^d$, the inner product $\dotp{\vec{a},\vec{b}} = \vec{a}^\t \vec{b}$ is the
usual coordinate dot product, and the second-moment operator is $\Sig =
\E(\vX\vX^\t)$. We assume that $\Sig$ is non-singular, so $L$ has a unique minimizer.
Here \algref{regression} can be used with $\lambda
= 0$. It is easy to see that \algref{regression} with the \textbf{variant} steps is a specialization of
\algref{main} with subsampled empirical loss minimization when $\ell =
\ell^\sq$, with the norm defined by $\norm{\vec{a}} = \sqrt{\vec{a}^\t \Sig_S \vec{a}}$. We now prove the guarantee for finite dimensional regression.

\begin{proof}[of \thmref{ols}]
The proof is derived from \corref{loss} as follows. 
First, suppose for simplicity that $\Sig_s = \Sig$, so that $\norm{\vec{a}} = \sqrt{\vec{a}^\t \Sig \vec{a}}$.
It is easy to check that $\norm{\cdot}_*$ is $1$-smooth, $\ell$ is
$R^2$-smooth with respect to $\norm{\cdot}$, and $L^\sq$ is $1$-smooth with
respect to $\norm{\cdot}$.
Moreover, consider a random sample $T$. By definition
\begin{equation*}
\frac{\vec\delta^\t \grad^2 L_T(\tilde\vw) \vec\delta}
{\norm{\vec\delta}^2}
= \frac{\vec\delta^\t \Sig_T \vec\delta}{\vec\delta^\t \Sig \vec\delta}
.
\end{equation*}
By \lemref{matrix-chernoff} with $\lambda = 0$, 
$\Pr[\inf\{ \vec{\delta}^\t \Sig_T \vec{\delta} / (\vec{\delta}^\t \Sig
\vec{\delta}) : \vec{\delta} \in \reals^d \setminus \{\vec0\} \} \geq 1/2]
\geq 5/6$, provided that $|T| \geq 80R^2\log(24|S|^2)$.
Therefore $n_{0.5} = O(R^2 \log R)$.
We can thus apply \corref{loss} with $\alpha = 0.5$, $\beta = R^2$,
$\bbeta = 1$, $\gamma = 1$, and $n_{0.5} = O(R^2 \log R)$, so with
probability at least $1-\delta$, the parameter $\hat\vw$ returned by
\algref{regression} satisfies
\begin{equation}
\label{eq:ols}
L(\hat\vw)
\leq \biggl( 1 + O\biggl( \frac{R^2 \log(1/\delta)}{n} \biggr) \biggr)
L(\vw_\opt),
\end{equation}
as soon as $n \geq O(R^2 \log(R) \log(1/\delta))$.

Now, by \lemref{matrix-chernoff}, if $n \geq O(R^2
\log(R/\delta))$, with probability at least $1-\delta$, the norm induced by $\Sig_S$ satisfies $(1/2)\vec{a}^\t \Sig
\vec{a} \leq \vec{a}^\t \Sig_S \vec{a} \leq 2\vec{a}^\t \Sig \vec{a}$ for
all $\vec{a} \in \reals^d$.
Therefore, by a union bound, the norm used by the algorithm is equivalent to the norm induced by the true $\Sig$ up to
constant factors, and thus leads to the same guarantee as given above
(where the constant factors are absorbed into the big-$O$ notation).
\end{proof}

The rate achieved in \eqref{ols} is well-known to be optimal up to
logarithmic factors~\citep{Nussbaum99}.
A standard argument for this, which we reference in the sequel, is as
follows.
Consider a distribution over $\reals^d \times \reals$ where $\vX \in
\reals^d$ is distributed uniformly over some orthonormal basis vectors
$\ve_1, \ve_2, \dotsc, \ve_d$, and $Y := \vX^\t \vw_\opt + Z$ for $Z \sim
\cN(0,\sigma^2)$ independent of $\vX$.
Here, $\vw_\opt$ is an arbitrary vector in $\reals^d$, $R = \sqrt{d}$, and the
optimal square loss is $L(\vw_\opt) = \sigma^2$.
Among $n$ independent copies of $(\vX,Y)$, let $n_i$ be the number of
copies with $\vX = \ve_i$, so $\sum_{i=1}^d n_i = n$. 
Estimating $\vw_\opt$ is equivalent to $d$ Gaussian mean estimation problems,
with a minimax loss of
\begin{align}
\inf_{\hat{\vw}}
\sup_{\vec{\vw_\opt}}
\E\bigl(L(\hat{\vw})\bigr) - L(\vw_\opt)
& \ = \
\inf_{\hat{\vw}}
\sup_{\vec{\vw_\opt}}
\E\biggl( \frac{1}{d}\norm{\hat{\vw} - \vw_\opt}_2^2 \biggr)
\notag \\
&
\ = \
\frac{1}{d}
\sum_{i=1}^d \frac{\sigma^2}{n_i}
\ \geq \
\frac{d\sigma^2}{n}
\ = \ \frac{d L(\vw_\opt)}{n}. 
  \label{eq:minimax}
\end{align}
Note that this also implies a lower bound for any estimator with
exponential concentration.
That is, for any estimator $\hat{\vw}$, if there is some $A > 0$ such that
for any $\delta \in (0,1)$, $\P[L(\hat{\vw}) > L(\vw_\opt) + A\log(1/\delta)]
< \delta$, then $A \geq \E(L(\hat{\vw})-L(\vw_\opt)) \geq dL(\vw_\opt) / n$.

\subsubsection{Ridge Regression}
\label{sec:infinite}

In a general, possibly infinite-dimensional, Hilbert space $\X$, \algref{regression} can be used with $\lambda > 0$.
In this case, \algref{regression} with the \textbf{variant} steps is again a specialization of \algref{main} with subsampled empirical loss minimization when $\ell = \ell^\lambda$, with the norm defined by $\norm{\vec{a}} = \sqrt{\vec{a}^\t (\Sig_S + \lambda \Id) \vec{a}}$.

\begin{proof}[of \thmref{ridge}]
As in the finite-dimensional case, assume first that $\Sig_S = \Sig$, and consider
the norm $\norm{\cdot}$ defined by $\norm{\vec{a}} := \sqrt{\dotp{\vec{a}, (\Sig + \lambda \Id) \vec{a}}}$.
It is easy to check that $\norm{\cdot}_*$ is $1$-smooth.
Moreover, since we assume that $\Pr[\dotp{\vX,\vX} \leq V^2] = 1$,
we have $\dotp{\vx,(\Sig + \lambda I)^{-1}\vx} \leq
\dotp{\vx,\vx} / \lambda$ for all $\vx \in \X$, so $\Pr[\dotp{\vX,(\Sig+\lambda
I)^{-1}\vX} \leq V^2/\lambda] = 1$.
Therefore $\ell^\lambda$ is $(1 + V^2/\lambda)$-smooth with respect to
$\norm\cdot$. In addition, $L^\lambda$ is $1$-smooth with respect to $\norm{\cdot}$.
Using \lemref{matrix-chernoff} with $r_\lambda = V/\lambda$, we have, similarly to the proof of \thmref{ols}, $n_{0.5} =
O((V^2/\lambda) \log(V/\sqrt{\lambda}))$.
Setting $\alpha = 0.5$, $\beta = 1+V^2/\lambda$, $\bbeta = 1$, $\gamma
= 1$, and $n_{0.5}$ as above, we conclude that with probability $1-\delta$,
\begin{equation*}
L^\lambda(\hat\vw)
\leq
\biggl( 1 + O\biggl( \frac{(1+V^2/\lambda)\log(1/\delta)}{n} \biggr) \biggr)
L^\lambda(\vw_\opt),
\end{equation*}
as soon as $n \geq O((V^2/\lambda)\log(V/\sqrt{\lambda})\log(1/\delta))$.
Again as in the proof of \thmref{ols}, by \lemref{matrix-chernoff}
\algref{regression} may use the observable norm $\vec{a} \mapsto
\dotp{\vec{a},(\Sig_S + \lambda I) \vec{a}}^{1/2}$ instead of the
unobservable norm $\vec{a} \mapsto \dotp{\vec{a},(\Sig + \lambda I)
\vec{a}}^{1/2}$ by applying a union bound, if $n \geq 
O((V^2/\lambda)\log(2V/(\delta\sqrt{\lambda})))$, losing only constant factors, .

We are generally interested in comparing to the minimum square loss
$L^\sq_\opt :=
\inf_{\vw \in \X} L^\sq(\vw)$, rather than the minimum regularized square loss
$\inf_{\vw \in \X} L^\lambda(\vw)$.
Assuming the minimizer is achieved by some $\bar\vw \in \X$ with
$\dotp{\bar\vw,\bar\vw} \leq B^2$, the choice
$\lambda = \Theta(\sqrt{L^\sq_\opt V^2\log(2/\delta) / (B^2n)})$
yields
\begin{equation*}
L^\sq(\hat\vw)
+ \lambda \dotp{\hat\vw,\hat\vw} \leq  L^\sq_\opt
+ O\biggl( \sqrt{\frac{L^\sq_\opt B^2V^2 \log(2/\delta)}{n}}
+ \frac{(L^\sq_\opt + B^2V^2)\log(2/\delta)}{n} \biggr)
\end{equation*}
as soon as $n \geq
\tilde{O}(B^2V^2\log(2/\delta)/L^\sq_\opt)$.
\end{proof}

By this analysis, a constant factor approximation for $L^\sq_\opt$ is achieved
with a sample of size $\tilde{O}(B^2V^2 \log(1/\delta) / L^\sq_\opt)$.
As in the finite-dimensional setting, this rate is known to be optimal up to
logarithmic factors~\citep{Nussbaum99}.
Indeed, a similar construction to that from \secref{finite} implies
\begin{equation}\label{eq:minimaxb}
\inf_{\hat{\vw}}
\sup_{\vec{\vw_\opt}}
\E\bigl( L(\hat{\vw})  - L(\vw_\opt))
\ \geq \
\Omega\biggl(
\frac{1}{d}\cdot\frac{L_\opt B^2V^2 \sum_{i=1}^d n_i^{-1}}{B^2V^2 +
L_\opt\sum_{i=1}^d n_i^{-1}}
\biggr)
\ \geq \
\Omega\biggl(
\frac{1}{d}\cdot\frac{L_\opt B^2V^2 d^2/n}{B^2V^2 + L_\opt d^2/n}
\biggr)
\end{equation}
(here, $\vX \in \{ V\ve_i : i \in [d] \}$ has Euclidean length $V$ almost
surely, and $B$ is a bound on the Euclidean length of $\vw_\opt$).
For $d = \sqrt{B^2V^2n/\sigma^2}$, the bound becomes
\[
\inf_{\hat{\vw}}
\sup_{\vec{\vw_\opt}}
\E\bigl( L(\hat{\vw}) - L(\vw_\opt))
\ \geq \
\Omega\biggl( \sqrt{\frac{L_\opt B^2V^2}{n}} \biggr) .
\]
As before, this minimax bound also implies a lower bound on any estimator
with exponential concentration.

\subsubsection{Heavy-tail Covariates}
\label{sec:heavyx}

When the covariates are not bounded or subgaussian, the empirical
second-moment matrix may deviate significantly from its population
counterpart with non-negligible probability. In this case it is not
possible to approximate the norm $\norm{\vec{a}} = \sqrt{\vec{a}^\t(\Sig + \lambda\Id)\vec{a}}$ in Step 2 of \algref{main} using a single small sample
(as discussed in \secref{finite} and \secref{infinite}).
However, we may use \algref{main2} instead of \algref{main}, which only
requires the stochastic distance measurements to be relatively accurate
with some constant probability. The full version of \algref{regression} is exactly such an implementation.

We now prove \thmref{heavyx}.
Define $c_\eta := 512(48c)^{2+2/\eta}(6+6/\eta)^{1+4/\eta}$ (which is
$C_{\text{\emph{main}}}$ from~\citealp{SV}).
The following lemma shows that $O(d)$ samples suffice so that the expected
spectral norm distance between the empirical second-moment matrix and
$\Sig$ is bounded.
\begin{lemma}[Implication of Corollary 1.2 from \citealp{SV}]
\label{lem:sv}
Let $\vX$ satisfy \condref{strong}, and let $\vX_1, \vX_2, \dotsc, \vX_n$ be
independent copies of $\vX$.
Let $\wh\Sig := \frac1n \sum_{i=1}^n \vX_i \vX_i^\t$.
For any $\eps \in (0,1)$, if $n \geq c_\eta \eps^{-2-2/\eta} d$, then
\[ \E \|\Sig^{-1/2} \wh\Sig \Sig^{-1/2} - \Id\|_2 \ \leq \ \eps . \]
\end{lemma}
\lemref{sv} implies that $n_{0.5} = O(c_\eta' d)$ where $c_\eta' = c_\eta \cdot
2^{O(1+1/\eta)}$.
Therefore, for $k = O(\log(1/\delta))$, subsampled empirical loss
minimization requires $n \geq k \cdot n_{0.5} = O(c_\eta' d
\log(1/\delta))$ samples to correctly implement $\APPROX_{\nrm,\veps}$, for
$\veps$ as in \eqref{loss-eps}.

Step \ref{step:sigma} in \algref{regression} implements $\DIST^j_{\nrm}$ as returning $f_j$ such that $f_j(\vv) := \norm{\Sig_{S_j}^{1/2}(\vv -\vw_j)}_2$. 
First, we show that \asref{dist1} holds.
By \lemref{sv}, an i.i.d.~sample $T$ of size $O(c_\eta'd)$
suffices so that with probability at least $8/9$, for every $\vv \in \reals^d$, 
\begin{align*}
(1/2) \|\Sig^{1/2}(\vv - \vw_j)\|_2 &\leq \|\Sig_T^{1/2}(\vv - \vw_j)\|_2 
\leq 2\|\Sig^{1/2}(\vv - \vw_j)\|_2 . 
\end{align*}
In particular, this holds for $T = S_j$, as long as $|S_j| \geq O(c_\eta'd)$. Thus, for $k = O(\log(1/\delta))$, \asref{dist1} holds if $n \geq O(c_\eta' d\log(1/\delta))$. \asref{dist2} (independence) also holds, since $f_j$ depends only on $S_j$, and $S_1,\ldots,S_k$ are statistically independent.

Putting everything together, we have (as in \secref{finite}) $\alpha =
0.5$ and $\gamma = 1$. We obtain the final bound from \thmref{loss2} as follows: if $n \geq O(c_\eta' d\log(1/\delta))$,
then with
probability at least $1-\delta$,
\begin{equation}
\label{eq:heavyx}
L(\hat\vw) - L(\vw_\opt)
= \norm{\Sig^{1/2}(\hat\vw - \vw_\opt)}_2^2
\leq O\Biggl( \frac{\E\norm{\Sig^{-1/2}\vX(\vX^\t\vw_\opt-Y)}_2^2 \log(1/\delta)}
{n}
\Biggr).
\end{equation}

\section{Other Applications}\label{sec:other}
In this section we show how the core techniques we discuss can be used for other applications, namely Lasso and low-rank matrix approximation.

\subsection{Sparse Parameter Estimation with Lasso}

In this section we consider $L^1$-regularized linear least squared
regression (Lasso) \citep{tibshirani96} with a random subgaussian design,
and show that \algref{main} achieves the same fast convergence rates for
sparse parameter estimation as Lasso, even when the noise is heavy-tailed.

Let $\cZ = \reals^d \times \reals$ and $\vw_\opt \in \reals^d$. Let $D$ be a distribution over $\cZ$, such that for $(\vX,Y) \sim D$, we have $Y = \vX^\t\vw_\opt + \veps$ where $\veps$ is an independent random variable with $\E[\veps] = 0$ and $\E[\veps^2] \leq \sigma^2$. We assume that $\vw_\opt$ is sparse: Denote the support of a vector $\vw$ by $\supp(\vw) := \{ j \in [d] : \vw_j \neq 0 \}$. Then $s := |\supp(\vw_\opt)|$ is assumed to be small compared to $d$. The \emph{design matrix} for a sample $S = \{(\vx_1,y_1),\ldots,(\vx_n,y_n)\}$ is an $l \times d$ matrix with the rows $\vx_i^\t$. 

For $\lambda > 0$, consider the Lasso loss
$\ell((\vx,y),\vw) = \frac12(\vx^\t\vw - y)^2 + \lambda\norm{\vw}_1$. 
Let $\nrm$ be the Euclidean norm in $\reals^d$. 
A random vector $\vX$ in $\reals^d$ is \emph{subgaussian} (with moment $1$) if for every vector $\vu \in \reals^d$, $\E[\exp(\vX^\t\vu)]\leq \exp(\norm{\vu}_2^2/2)$.

The following theorem shows that when \algref{main} is used with subsampled empirical loss minimization over the Lasso loss, and $D$ generates a subgaussian random design, then $\vw$ can be estimated for any type of noise $\veps$, including heavy-tailed noise. 

In order to obtain guarantees for Lasso the design matrix must satisfy some regularity conditions.  
We use the \emph{Restricted Eigenvalue condition} (RE)  proposed in \cite{bickel09}, which we presently define.
For $\vw \in \reals^d$ and $J \subseteq [d]$, let $[\vw]_J$ be the $|J|$-dimensional vector which is equal to $\vw$ on the coordinates in $J$. Denote by $\vw_{[s]}$ the $s$-dimensional vector with coordinates equal to the $s$ largest coordinates (in absolute value) of $\vw$. Let $\vw_{[s]^C}$ be the $(d-s)$-dimensional vector which includes the coordinates not in $\vw_{[s]}$.
Define the set $E_s = \{ \vu \in \reals^d\setminus \{0\} \mid \norm{\vu_{[s]^C}}_1 \leq 3\norm{\vu_{[s]}}_1\}$. For an $l\times d$ matrix $\Psi$ (for some integer $l$), let $\gamma(\Psi, s) = \min_{\vu\in E_s} \frac{\norm{\Psi\vu}_2}{\norm{\vu_{[s]}}_2}.$ The RE condition for $\Psi$ with sparsity $s$ requires that  $\gamma(\Psi, s) > 0$. We further denote $\eta(\Psi,s) = \max_{\vu \in \reals^d \setminus \{0\}: |\supp(\vu)| \leq s} \frac{\norm{\Psi\vu}_2}{\norm{\vu}_2}.$

\begin{theorem}\label{thm:lassorandom}
Let $C, c >0$ be universal constants. Let $\Sig \in \reals^{d\times d}$ be a positive semi definite matrix.
Denote $\eta := \eta(\Sig^\half,s)$ and $\gamma := \gamma(\Sig^\half,s)$.
Assume the random design setting defined above, with $\vX =  \Sig^\half \vec{Z} $, where $\vec{Z}$ is a subgaussian random vector.
Suppose \algref{main} uses subsampled empirical loss minimization with the empirical Lasso loss, with $\lambda = 2 \sqrt{\sigma^2\eta^2\log (2d)\log(1/\delta)/n}$.
If $n \geq c s \frac{\eta^2}{\gamma^2} \log(d)\log(1/\delta)$, then with probability $1-\delta$,
The vector $\hat{\vw}$ returned by \algref{main} satisfies 
\begin{align*}
&\norm{\hat{\vw} - \vw_\opt}_2 \leq
\frac{C\sigma\eta}{\gamma^2}\sqrt{\frac{s\log (2d)\log(1/\delta)}{n}}.
\end{align*}

\end{theorem}

For the proof of \thmref{lassorandom}, we use the following theorem, adapted from \cite{bickel09} and \cite{zhang09}. The proof is provided in \appref{proofs} for completeness.

\begin{theorem}[\cite{bickel09,zhang09}] \label{thm:lasso}
Let $\Psi = [\Psi_1|\Psi_2|\ldots|\Psi_d] \in \reals^{n \times d}$ and $\vec{\veps} \in \reals^n$. Let $y = \Psi \vw_\opt + \vec{\veps}$ and $\hat{\vw} \in \argmin_\vw \frac{1}{2}\norm{\Psi\vw - y}_2^2 + \lambda \norm{\vw}_1$. Assume that $|\supp(\vw_\opt)| = s$ and that $\gamma(\Psi, s) > 0$.
If $\norm{\Psi^\t \vec{\varepsilon}}_\infty \leq \lambda/2$,
then 
\begin{align*}
\norm{\hat{\vw} - \vw_\opt}_2 &\leq \frac{12\lambda\sqrt{s}}{\gamma^2(\Psi,s)}.
\end{align*}
\end{theorem}

\begin{proof}[of \thmref{lassorandom}]
Fix $i \in [k]$, and let $n_i = n/k$. Let $\Psi \in \reals^{n_i \times d}$
be the design matrix for $S_i$ and let $\vw_i$ be the vector returned by
the algorithm in round $i$, $\vw_i \in \argmin \frac{1}{2n}\norm{\Psi \vw -
\vy}_2^2 + \lambda\norm{\vw}_1$.
It is shown in \cite{Zhou09} that if $n_i \geq C \frac{\eta^2}{\gamma^2}s
\log(d)$ for a universal constant $C$, then with probability $5/6$,
$\min_{\vu \in E_s}\frac{\norm{\Psi \vu}_2}{\norm{\Sig^{\half} \vu}_2} \geq \sqrt{n_i}/2$. Call this event $\cE$.
By the definition of $\gamma$, we have that under $\cE$, 
\[
\gamma(\Psi, s) = \min_{\vu \in E_s}\frac{\norm{\Psi \vu}_2}{\norm{\vu_{[s]}}_2} = 
\min_{\vu \in E_s}\frac{\norm{\Psi \vu}_2}{\norm{\Sig^{\half}
\vu}_2}\frac{\norm{\Sig^{\half} \vu}_2}{\norm{\vu_{[s]}}_2} \geq 
\sqrt{n}\,\gamma/2.
\]

If $\cE$ holds and $\norm{\Psi^\t \vec{\veps}}_\infty \leq n\lambda/2$,
then we can apply \thmref{lasso} (with $n\lambda$ instead of $\lambda$). We
now show that this inequality holds with a constant probability. Fix the
noise vector $\vec{\varepsilon} = \vec{y} - \Psi \vw_\opt$. For $l \in [d]$,
since the coordinates of $\vec{\veps}$ are independent and each row of
$\Psi$ is an independent copy of the vector $\vX = \Sig^{\half}\vec{Z}$, we have
\[
\E[\exp([\Psi^\t \vec{\varepsilon}]_l) \mid \vec{\veps}] =
\prod_{j\in[n]}\E[\exp(\Psi_{j,l}\vec{\varepsilon}_j)\mid \vec{\veps}] =
\prod_{j\in[n]}\E[\exp(\vec{Z}(\vec{\varepsilon}_j\Sig^{\half}\vec{e}_l))\mid \vec{\veps}].
\]
Since $\norm{\vec{\varepsilon}_j\Sig^{\half}\vec{e}_l}_2 \leq \vec{\varepsilon}_j \eta$, 
we conclude that 
\[
\E[\exp([\Psi^\t \vec{\varepsilon}]_l) \mid \vec{\veps}] \leq \prod_{j\in[n]}\exp( \vec{\varepsilon}_j^2/2) = \exp(\eta^2\norm{\vec{\varepsilon}}_2^2/2).
\]
Therefore, for $\xi > 0$
\begin{align*}
\xi\E[\norm{\Psi^\t \vec{\varepsilon}}_\infty \mid \vec{\veps}] &= \E[\max_l(\xi| [\Psi^\t \vec{\varepsilon}]_l|)\mid \vec{\veps}] = \E[\log \max_l \exp(\xi|[\Psi^\t \vec{\varepsilon}]_l|)\mid \vec{\veps}]\\
& \leq \E[\log \left(\sum_l \exp(\xi[\Psi^\t \vec{\varepsilon}]_l)+\exp(-\xi[\Psi^\t \vec{\varepsilon}]_i)\right)\mid \vec{\veps}]\\
&\leq \log \left(\sum_l \E[\exp(\xi[\Psi^\t \vec{\varepsilon}]_l)\mid \vec{\veps}] + \E[\exp(-\xi[\Psi^\t \vec{\varepsilon}]_l)\mid \vec{\veps}]\right) \\
&\leq \log (2d) + \xi^2 \eta^2\norm{\vec{\varepsilon}}_2^2/2.
\end{align*}
Since $\E[\vec{\varepsilon}_j^2] \leq \sigma^2$ for all $j$, we have $\E[\norm{\vec{\varepsilon}}^2] \leq n_i \sigma^2/2$. 
Therefore \[
\E[\norm{\Psi^\t \vec{\varepsilon}}_\infty] \leq \frac{\log (2d)}{\xi} + \xi n_i \eta^2\sigma^2/2.
\]
Minimizing over $\xi > 0$ we get $\E[\norm{\Psi^\t \vec{\varepsilon}}_\infty] \leq 2\sqrt{\sigma^2 \eta^2\log (2d)n_i/2}$.
therefore by Markov's inequality, with probability at least $5/6$, $\frac{1}{n_i}\norm{\Psi^\t \vec{\varepsilon}}_\infty \leq 2\sqrt{\sigma^2\eta^2\log (2d)/n_i} = \lambda$. With probability at least $2/3$ this holds together with $\cE$.

In this case, by \thmref{lasso},
\begin{align*}
&\norm{\vw_i - \vw_\opt}_2 \leq \frac{12\lambda\sqrt{s}}{\gamma^2(\Psi,s)} \leq 
\frac{24}{\gamma^2} \sqrt{\frac{s\sigma^2\eta^2\log (2d)}{n_i}}.
\end{align*}
Therefore $\APPROX_{\nrm,\epsilon}$ satisfies \asref{approx1} with $\epsilon$ as in the right hand side above.
The statement of the theorem now follows by applying \propref{alg} with $k = O(\log(1/\delta)$, and noting that $n_i = O(n/\log(1/\delta))$.
\end{proof}

It is worth mentioning that we can apply our technique to the fixed design
setting, where design matrix $X \in \reals^{n \times d}$ is fixed and not
assumed to come from any distribution.
If $X$ satisfies the RE condition, as well as a certain low-leverage
condition---specifically, that the \emph{statistical leverage
scores}~\citep{chatterjee1986influential} of any $n \times O(s)$ submatrix of $X$ be
roughly $O(1/(ks\log d))$---then \algref{main} can be used with the
subsampled empirical loss minimization implementation of
$\APPROX_{\nrm,\veps}$ to obtain similar guarantees as in the random
subgaussian design setting.

We note that while standard analyses of sparse estimation with
mean-zero noise assume light-tailed noise \citep{zhang09,bickel09},
there are several works that analyze sparse estimation with
heavy-tailed noise under various assumptions.
For example, several works assume that the median of the noise is zero
(\emph{e.g.},
\citealt{wang2013l1,belloni2011l1,zou2008composite,wu2009variable,wang2007robust,fan2012adaptive}).
\citet{van2012quasi} analyze a class of optimization functions that
includes the Lasso and show polynomial convergence under fourth-moment
bounds on the noise.
\cite{chatterjee2013rates} study a two-phase sparse estimator for
mean-zero noise termed the Adaptive Lasso, proposed in
\cite{zou2006adaptive}, and show asymptotic convergence results under
mild moment assumptions on the noise.

\subsection{Low-rank Matrix Approximation}

The proposed technique can be easily applied also to low-rank covariance matrix
approximation for heavy tailed distributions. Let $\cD$ be a
distribution over $\cZ = \reals^d$ and suppose our goal is to estimate
$\Sig = \E[\vX \vX^\t]$ to high accuracy, assuming that $\Sig$ is
(approximately) low rank. Here $\X$ is the space of $\reals^{d\times d}$
matrices, and $\norm{\cdot}$ is the spectral norm. Denote the Frobenius
norm by $\norm{\cdot}_F$ and the trace norm by $\norm{\cdot}_{\tr}$. For $S
= \{\vX_1,\ldots,\vX_n\} \subseteq \reals^d$, define the empirical
covariance matrix $\Sig_S = \frac{1}{n}\sum_{i\in[n]}\vX_i\vX_i^\t$.
We have the following result for low-rank estimation:
\begin{lemma}[\citealt{KLT11}]\label{lem:kolt}
Let $\hat{\Sig} \in \reals^{d\times d}$. Assume $\lambda \geq
\norm{\hat{\Sig} - \Sig}$, and let 
\begin{equation}\label{eq:sigma}
\Sig_\lambda \in \argmin_{A \in \reals^{d\times d}}  \half\norm{\hat{\Sig}
- A}_F^2 + \lambda \norm{A}_{\tr},
\end{equation}
If $\lambda \geq \norm{\hat{\Sig} - \Sig}$, then
\begin{align*}
\frac12\|\hat{\Sig}_\lambda - \Sig\|_F^2
& \leq
\inf_{A \in \reals^{d\times d}} \biggl\{
\frac12\|A - \Sig\|_F^2
+ \frac12 (\sqrt{2} + 1)^2 \lambda^2 \rank(A)
\biggr\}.
\end{align*}
\end{lemma}

Now, assume condition \ref{cond:strong} holds for $\cX \sim \cD$, and suppose
for simplicity that $\norm{\Sig}\leq 1$. In this case, by \lemref{sv}, 
A random sample $S$ of size $n' = c'_\eta \epsilon^{-2-2/\eta}d$, where
$c'_\eta = c_\eta (3/2)^{2+2/\eta}$ suffices to get an empirical covariance
matrix $\Sig_S$ such that $\norm{\Sig_S - \Sig} \leq \epsilon$ with probability at least $2/3$.

Given a sample of size $n$ from $\cD$,
We can thus implement $\APPROX_{\nrm,\veps}$ that simply returns the empirical covariance matrix of a sub-sample of size $n' = n/k$, so that 
\asref{approx1} holds for an appropriate $\veps$.
By \propref{alg}, \algref{main} returns $\hat{\Sig}$ such that with
probability at least $1-\exp(-k/18)$, $\norm{\hat{\Sig} - A} \leq 3\epsilon$. 
The resulting $\hat{\Sig}$ can be used to minimize \eqref{sigma} with
$\lambda = 3\epsilon := O\left((c'_\eta
d\log(1/\delta)/n)^{1/2(1+1/\eta)}\right)$. The output matrix $\Sig_\lambda$ satisfies, with probability at least $1-\delta$, 
\[
\frac12\|\Sig_\lambda - \Sig\|_F^2 \leq \inf_{A \in \reals^{d\times d}} \biggl\{
\frac12\|A - \Sig\|_F^2
+ O\left((c'_\eta d\log(1/\delta)/n)^{1/(1+1/\eta)}\right)\cdot \rank(A)
\biggr\}.
\]

\section{A Comparison of Robust Distance Approximation Methods}
\label{sec:geometric}

The approach described in \secref{core} for selecting a single $\vw_i$ out of the set $\vw_1,\ldots,\vw_k$, gives one Robust Distance Approximation procedure (see \defref{rda}), in which the $\vw_i$ with the lowest median distance from all others is selected. In this section we consider other Robust Distance Approximation procedures and their properties. We distinguish between procedures that return $y \in W$, which we term \emph{set-based}, and procedures that might return any $y \in \X$, which we term \emph{space-based}.

Recall that we consider a metric space $(\X,\rho)$, with $W \subseteq \X$ a (multi)set of size $k$ and $w_\opt$ a distinguished element. Let $W_+ := W\cup \{w_\opt\}$.
In this formalization, the procedure used in \algref{main} is to
simply select $y \in \argmin_{w \in W}\omed_W(w,0)$, a set-based
procedure. A natural variation of this is the space-based procedure:
select $y \in \argmin_{w \in \X}\omed_W(w,0)$.\footnote{The
  space-based median distance approach might not
always be computationally feasible; see discussion in
\secref{compmedian}.} A different approach, proposed by
\citet{Minsker13}, is to select $y \in \argmin_{w \in \X}\sum_{\bar{w}
\in W}\rho(w,\bar{w})$, that is to minimize the geometric median over
the space. \citeauthor{Minsker13} analyzes this approach for Banach
and Hilbert spaces. We show that minimizing the geometric median also achieves similar guarantees in general metric spaces.

In the following, we provide detailed guarantees for the approximation
factor $C_\alpha$ of the two types of procedures, for general metric
spaces as well as for Banach and Hilbert spaces, and for set-based and
sample-based procedures. We further provide lower bounds for specific
procedures, as well as lower bounds that hold for any procedure. 
In \secref{compmedian} we summarize the results and compare the guarantees of the two procedures and the lower bounds.
For a more useful comparison, we take into account the fact that the value of $\alpha$ usually affects not only the approximation factor, but also the upper bound obtained  for $\omed_W(w_\opt,\alpha)$.

\subsection{Minimizing the Median Distance}

Minimizing the median distance over the set of input points was shown in \propref{genmed} to achieve an approximation factor of 3. In this section we show that this upper bound on the approximation factor is tight for this procedure, even in a Hilbert space. Here and below, we say that an approximation factor upper bound is \emph{tight} if for any constant smaller than this upper bound, there are a suitable space and a set of points in that space, such that the procedure achieves for this input a larger approximation factor than said constant.

The approximation factor can be improved to $2$ for a sample-based procedure. This factor is tight as well, even assuming a Hilbert space. The following theorem summarizes these facts.

\begin{theorem}\label{thm:meddist}
Let $k \geq 2$, and suppose that $\omed_W(w_\opt, \gamma) \leq \epsilon$ for some $\gamma > 0$.
Let $y \in \argmin_{w \in W} \omed_W(w, 0)$. Further, suppose that $W_+ \subseteq \X$, and let $\bar{y} \in \argmin_{w \in \X}\omed_W(w, 0)$. 
Then
\begin{itemize}
\item For any metric space, $\rho(w_\opt, y) \leq 3\epsilon$;
\item For any metric space, $\rho(w_\opt, \bar{y}) \leq 2\epsilon$;
\item There exists a set on the real line such that $\rho(w_\opt, y) = 3\epsilon$, where $\rho$ is the distance induced by the inner product;
\item There exists a set on the real line such that $\rho(w_\opt, \bar{y}) = 2\epsilon$, where $\rho$ is the distance induced by the inner product.
\end{itemize}
\end{theorem}

\begin{proof}
First, we prove the two upper bounds.
Since $\omed_W(w_\opt, \gamma) \leq \epsilon$, we have $|B(w_\opt, \epsilon) \cap W| > k/2$. Let $w \in |B(w_\opt, \epsilon) \cap W|$. Then by the triangle inequality,
$B(w,2\epsilon) \supseteq B(w_\opt,\epsilon)$. Therefore $\omed_W(w,0) \leq 2\epsilon$. It follows that $\omed_W(y,0) \leq 2\epsilon$, hence $|B(y, 2\epsilon) \cap W| \geq k/2$. By the pigeon hole principle, $|B(w_\opt, \epsilon) \cap B(y,2\epsilon)| > 0$, therefore $\rho(w_\opt,y) \leq 3\epsilon$. 

As for $\bar{y}$, since this is a minimizer over the entire space $\X$ which includes $w_\opt$, we have $\omed_W(y,\gamma) \leq \omed_W(w_\opt, \gamma) \leq \epsilon$. Therefore, similarly to the argument for $y$, we have $\rho(w_\opt,y) \leq 2\epsilon$. 

To see that these bounds are tight, we construct simple examples on the real line. For $y$, suppose $w_\opt = \epsilon$, and consider $W$ with $k$ points as follows: $k/2-1$ points at $0$, $2$ points at $2\epsilon$, and $k/2-1$ points at $4\epsilon$. The points at $4\epsilon$ are clearly in $\argmin_{w \in W} \omed_W(w, 0)$, therefore $\rho(w_\opt,y) = 3\epsilon$.

For $\bar{y}$, suppose $w_\opt = \epsilon$, and consider $W$ with $k$ points as follows: $2$ points at $0$, $k/2-1$ points at $2\epsilon$, and $k/2-1$ points at $3\epsilon$. The points at $3\epsilon$ are clearly in $\argmin_{w \in W_+} \omed_W(w, 0)$, therefore $\rho(w_\opt,\bar{y}) = 2\epsilon$.
\end{proof}
 
The non-uniqueness of the median distance minimizer is exploited in
the lower bounds in \thmref{meddist}.
This suggests that some kind of aggregation of the median distance
minimizers may provide a smaller bound at least in certain scenarios.

\subsection{The Geometric Median}

For $w \in \X$, denote the sum of distances from points in the input set by $\gmed(w) := \sum_{v \in W} \rho(w,v)$. \cite{Minsker13} suggests to minimize the sum of distances over the entire space, that is, to select the geometric median. Minsker shows that when this procedure is applied in a Hilbert space, $C_\alpha \leq \frac{\half + \alpha}{\sqrt{2\alpha}}$, and for a Banach space $C_\alpha \leq 1+\frac{1}{2\alpha}$. 
Here we show that in fact $C_\alpha \leq 1+\frac{1}{2\alpha}$ for general metric spaces. The proof holds, in particular, for Banach spaces, and thus this provide a more direct argument that does not require the special properties of Banach spaces. We further show that for general metric spaces, this upper bound on the approximation factor is tight.

Minimizing over the entire space is a computationally intensive procedure, involving convex approximation. Moreover, if the only access to the metric is via estimated distances based on samples, as in \algref{main2}, then there are additional statistical challenges. It is thus of interest to also consider the simpler set-based procedure, and we provide approximation guarantees for this procedure as well. We show that an approximation factor of $2+\frac{1}{2\alpha}$ can be guaranteed for set-based procedures in general metric spaces, and this is also tight, even for Banach spaces. 

The following theorem provides a bound that holds in several of these settings.
\begin{theorem}
Let $k \geq 2$. Let $y \in \argmin_{w\in W} \gmed(w)$, and let $\bar{y} \in \argmin_{w \in W_+}\gmed(w)$. Then 
\begin{enumerate}
\item For any metric space $(\X,\rho)$ and $W,W_+$, \label{item:upperset}
\[
  \rho(w_\opt, y) \leq \left(2+\frac{1}{2\alpha}\right)
  \omed_W(w_\opt, \alpha).
\]
\item For any constant $C < (2+\frac{1}{2\alpha})$, there exists a problem in a Banach space such that $\rho(w_\opt, y) > C\cdot\omed_W(w_\opt, \alpha)$. Thus the upper bound above is tight. \label{item:lowerset}
\item For any metric space $(\X,\rho)$ and $W,W_+$,\label{item:upperspace}
\[
  \rho(w_\opt, \bar{y}) \leq \left(1+\frac{1}{2\alpha}\right)
  \omed_W(w_\opt, \alpha).
\]
\item For any constant $C < (1+\frac{1}{2\alpha})$, there exists a problem in a metric space such that $\rho(w_\opt, \bar{y}) > C\cdot\omed_W(w_\opt, \alpha)$. Thus the upper bound above is tight for general metric spaces.\label{item:lowerspace}
\end{enumerate}
\end{theorem}
\begin{proof}
 Let $w \in \argmin_{w \in \ball(w_\opt,\epsilon)\cap W} \rho(w,y)$.
Let $Z \subset \ball(w_\opt,\epsilon)\cap W$ such that $|Z| = k(\half + \alpha)$ (we assume for simplicity that $k(\half + \alpha)$ is an integer; the proof can be easily modified to accommodate the general case). For $v \in Z$, $\rho(w,v) \leq \rho(w,w_\opt) + \rho(w_\opt, v)$. 
For $v \in W \setminus Z$, $\rho(w,v) \leq \rho(w,y) + \rho(y, v)$. Therefore
\[
\gmed(w) \leq \sum_{v \in Z} (\rho(w,w_\opt) + \rho(w_\opt, v)) + \sum_{v \in W \setminus Z} (\rho(w,y) + \rho(y,v)).
\]
By the definition of $w$ as a minimizer, for $v \in Z$, $\rho(y,v) \geq \rho(y,w)$. Thus
\[
\gmed(y) \geq \sum_{v \in Z} \rho(y,w)  + \sum_{v \in W \setminus Z}\rho(y,v) .
\]
Since $\gmed(y) \leq \gmed(w)$, we get
\[
\sum_{v \in Z} \rho(y,w)  + \sum_{v \in W \setminus Z}\rho(y,v) \leq 
\sum_{v \in Z} (\rho(w,w_\opt) + \rho(w_\opt, v)) + \sum_{v \in W \setminus Z} (\rho(w,y) + \rho(y,v)).
\]
Hence, since $\rho(v,w_\opt) \leq \epsilon$ for $v \in Z$, 
\[
(|Z| - |W \setminus Z|)\rho(w,y) \leq 2|Z|\epsilon.
\]
Since $|Z| = k(\half + \alpha)$ it follows that $\rho(w,y) \leq (1 + \frac{1}{2\alpha})\epsilon$. In addition, 
\[
\rho(w_\opt,y) \leq \rho(w,w_\opt) + \rho(w,y) \leq \epsilon + \rho(w,y),
\] 
therefore
\[
  \rho(w_\opt,y) \leq \left(2 + \frac{1}{2\alpha}\right)\epsilon.
\]
This shows that for any metric space, the set-based geometric median gives an approximation factor of $2+\frac{1}{2\alpha}$, proving item \ref{item:upperset}.

For the space-based geometric median, consider $\bar{w} \in \argmin_{w \in \ball(w_\opt,\epsilon)\cap W} \rho(w,\bar{y})$. We have $\gmed(\bar{y}) \leq \gmed(w_\opt)$. In addition,
\[
\gmed(w_\opt) \leq \sum_{v \in Z} \rho(w_\opt, v) + \sum_{v \in W \setminus Z} (\rho(w_\opt,\bar{w}) + \rho(\bar{w},\bar{y}) + \rho(\bar{y},v)).
\]
Therefore,
\[
\sum_{v \in Z} \rho(\bar{y},\bar{w})  + \sum_{v \in W \setminus Z}\rho(\bar{y},v) \leq \sum_{v \in Z} \rho(w_\opt, v) + \sum_{v \in W \setminus Z} (\rho(w_\opt,w) + \rho(w,\bar{y}) + \rho(\bar{y},v)).
\]
Since $\rho(w_\opt,v) \leq \epsilon$ for $v \in Z$, and $\rho(w_\opt,\bar{w}) \leq \epsilon$, it follows 
\[
(|Z| - |W \setminus Z|)\rho(\bar{w},\bar{y}) \leq k\epsilon.
\]
Therefore $\rho(\bar{w},\bar{y}) \leq \frac{1}{2\alpha}\epsilon$, hence 
\[
\rho(w_\opt,\bar{y}) \leq \rho(w_\opt, \bar{w}) + \rho(\bar{w},\bar{y}) \leq
\left(1+\frac{1}{2\alpha}\right)\epsilon.
\]
This gives an approximation factor of $1+\frac{1}{2\alpha}$ for space-based geometric median, proving item \ref{item:upperspace}.

To see that both of these bounds are tight, let $n = k(\half + \alpha)$, and let $\X = W_+ = \{v_1,\ldots, v_n, y_1,\ldots, y_{k-n}, w_\opt\}$. Define $\rho(\cdot,\cdot)$ as follows (for all pairs $i \neq j$, $l \neq t$):
\begin{align*}
&\rho(w_\opt, v_i) = \epsilon\\
&\rho(w_\opt, y_l) = \beta\\
&\rho(v_i, v_j) = 2\epsilon\\
&\rho(v_i, y_t) = \beta - \epsilon\\
&\rho(y_t, y_l) = 0.
\end{align*}
One can verify that for any $\beta \leq (2 + \frac{1}{2\alpha} - \frac{1}{k\alpha})\epsilon$, $\gmed(y_l) \leq \gmed(v_i)$ for all $l,i$. 
Therefore, the approximation factor for set-based geometric median in a general metric space is lower-bounded by $2+\frac{1}{2\alpha}$ for general $k$. 
This holds also for Banach spaces as well, Since any metric space can be embedded into a Banach space \citep{Kuratowski35}. This proves item \ref{item:lowerset}.

For space-based geometric median, note that if $\beta \leq (1+\frac{1}{2\alpha})\epsilon$, then $\gmed(w_\opt) \geq \gmed(y_l)$. Therefore the space-based upper bound is tight for a general metric space. This proves item \ref{item:lowerspace}.
\end{proof}

Since $\alpha \in (0,\half)$, the guarantee for the geometric median in these settings is always worse than the guarantee for minimizing the median distance. Factoring in the dependence on $\alpha$, the difference is even more pronounced. The full comparison is given in \secref{compmedian} below.

\subsection{Optimal Approximation Factor}
In this section we give lower bounds that hold for any robust distance
approximation procedure. A lower bound of $C > 0$ for a category of metric spaces and a type of procedure indicates that if a procedure of this type guarantees a distance approximation $C_\alpha$ for all metric spaces of the given category, then necessarily $C_\alpha \geq C$. As shown below, in many cases the lower bounds provided here match the
upper bounds obtained by either the median distance or the geometric
median. 

The following theorem gives a lower bound of $3$ for the achievable approximation factor of set-based procedures in Banach spaces (and so, also in general metric spaces). This factor is achieved by the median distance minimizer, as shown in \thmref{meddist}. 
\begin{theorem}\label{thm:setmetric}
Consider set-based robust distance approximation procedures. For any $\alpha \in (0,\half)$, and for any such procedure, there exists a problem in a Banach space for which the approximation factor of the procedure is at least $3$. 
\end{theorem}

\begin{proof}
Fix $\alpha$, and let $n = \ceil{\frac{1}{\half - \alpha}}$. 
Define the metric space $\X = \{a_1,\ldots,a_n,b_1,\ldots,b_n\}$ with
the metric $\rho(\cdot,\cdot)$ defined as follows: For all $i \neq j$, $\rho(a_i,a_j) = 2$,
$\rho(a_i, b_j) = 1$, $\rho(b_i,b_j) = 2$. For all $i$, $\rho(a_i,b_i) = 3$. See \figref{graph} for illustration.

\begin{figure}[h]
\begin{center}
\begin{tikzpicture}[scale = 3]
\draw (-0.5,0) -- (0.5,0) -- (0,0.866) -- cycle;
\fill (-0.5,0) circle(0.5pt) node[anchor=north] {$a_1$};
\fill (0.5,0) circle(0.5pt) node[anchor=west] {$a_2$};
\fill (0,0.866) circle(0.5pt) node[anchor=west] {$a_3$};
\fill (-0.25,0.433) circle(0.5pt)node[anchor=west] {$b_2$};
\fill (0.25,0.433) circle(0.5pt)node[anchor=west] {$b_1$};
\fill (0,0) circle(0.5pt)node[anchor=north] {$b_3$};
\end{tikzpicture}
\caption{The metric defined in \thmref{setmetric} for $n=3$. The distances are shortest paths on the underlying undirected graph, where all edges are the same length.}
\label{fig:graph}
\end{center}
\end{figure}
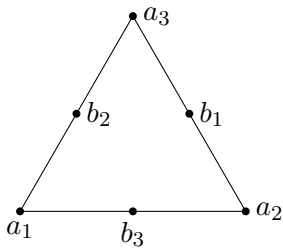

Consider the multi-set $W$ with $k/n$ elements at every $b_i$. It is easy to check that for every $a_i$, $\omed_W(a_i,\alpha)\leq  \omed_W(a_i,1/2-1/n) = 1$. On the other hand, since the problem is symmetric for permutations of the indices $1,\ldots,n$, 
no procedure can distinguish the cases $w_\opt = a_i$ for different $i \in [n]$.
For any choice $y = b_i \in W$, if $w_\opt = a_i$ then $\rho(w_\opt,y) = 3$. Therefore the approximation factor of any procedure is at least $3$. 
Since any metric space can be embedded into a Banach space \citep{Kuratowski35} this result holds also for Banach spaces.
\end{proof}

Next, we give a lower bound of $2$ for space-based procedures over general metric spaces. \thmref{meddist} shows that this factor is also achieved by minimizing the median distance.
\begin{theorem}
Consider robust space-based distance approximation procedures. For any $\alpha \in (0,\half)$, and for any such procedure, there exists a problem for which the approximation factor of the procedure is at least $2$. 
\end{theorem}

\begin{proof}
Fix $\alpha$, and let $n = \ceil{\frac{1}{\half - \alpha}}$. 
Define the metric space $\X = \{a_1,\ldots,a_n,b_1,\ldots,b_n\}$ with
the metric $\rho(\cdot,\cdot)$ defined as follows: For all $i \neq j$, $\rho(a_i,a_j) = 2$,
$\rho(a_i, b_j) = 1$, $\rho(b_i,b_j) = 1$. For all $i$, $\rho(a_i,b_i) = 2$. See \figref{graph2} for illustration.

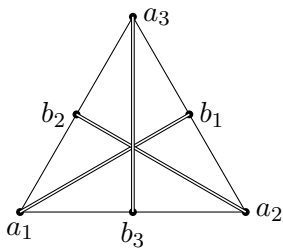
\begin{figure}[h]
\begin{center}
\begin{tikzpicture}[scale = 3]
\draw (-0.5,0) -- (0.5,0) -- (0,0.866) -- cycle;
\fill (-0.5,0) circle(0.5pt) node[anchor=north] {$a_1$};
\fill (0.5,0) circle(0.5pt) node[anchor=west] {$a_2$};
\fill (0,0.866) circle(0.5pt) node[anchor=west] {$a_3$};
\fill (-0.25,0.433) circle(0.5pt)node[anchor=east] {$b_2$};
\fill (0.25,0.433) circle(0.5pt)node[anchor=west] {$b_1$};
\fill (0,0) circle(0.5pt)node[anchor=north] {$b_3$};
\draw[double] (-0.5,0) -- (0.25,0.433);
\draw[double] (0.5,0) -- (-0.25,0.433);
\draw[double] (0,0.866) -- (0,0);
\end{tikzpicture}
\caption{The metric defined in \thmref{setmetric} for $n=3$. The distances are shortest paths on the underlying undirected graph. The full lines are edges of length $1$, the double lines from $a_i$ to $b_i$ are edges of length $2$.}
\label{fig:graph2}
\end{center}
\end{figure}

Consider the multi-set $W$ with $k/n$ points at every $b_i$. It is easy to check that for every $a_i$, $\omed_W(a_i,\alpha)\leq  \omed_W(a_i,1/2-1/n) = 1$. On the other hand, since the problem is symmetric for permutations of the indices $1,\ldots,n$, 
no procedure can distinguish the cases $w_\opt = a_i$ for different $i \in [n]$. Moreover, any point $y$ in the space has $\rho(a_i,y) = 2$ for at least one $i \in [n]$. Therefore the approximation factor of any procedure is at least $2$. 
\end{proof}

For lower bounds on Hilbert spaces and Banach spaces, we require the following lemma, which gives the radius of the ball inscribing the regular simplex in a $p$-normed space.
\begin{lemma}\label{lem:tetra}
Consider $\reals^n$ with the $p$-norm for $p > 1$. Let $e_1,\ldots,e_n$ be the standard basis vectors, and let $r_{n,p}$ be the minimal number for which there exists an $x \in \reals^n$ such that $B(x,r) \supseteq \{e_1,\ldots,e_n\}$. 
Then $r_{n,p} = ((1+(n-1)^{-1/(p-1)})^{-p} + (n-1)(1+(n-1)^{1/(p-1)})^{-p})^{1/p}$.
This radius is obtained with the center $x$ such that for all $i$, $x_i = (1+(n-1)^{1/(p-1)})^{-1}$.
\end{lemma}
\begin{proof}
It is easy to see that due to symmetry, $x = (a,a,\ldots,a)$ for some real number $a$. Thus $r_{n,p} = \inf_{a \in \reals}\norm{e_1 - (a,\ldots,a)}_p$.
We have $\norm{e_1 - (a,\ldots,a)}^p_p = |1-a|^p + (n-1)|a|^p$.
Minimizing over $a$ gives $a = (1+(n-1)^{1/(p-1)})^{-1}$, and 
\[
r_{n,p}^p = |1-a|^p + (n-1)|a|^p = (1+(n-1)^{-1/(p-1)})^{-p} + (n-1)(1+(n-1)^{1/(p-1)})^{-p}.
\]
\end{proof}

We now prove a lower bound for robust distance approximation in Hilbert spaces. Unlike the previous lower bounds, this lower bound depends on the value of $\alpha$.

\begin{theorem}\label{thm:hilbertlower}
Consider robust distance approximation procedures for $(\X,\rho)$ a Hilbert space. For any $\alpha \in (0,\half)$, the following holds:
\begin{itemize}
\item For any set-based procedure, there exists a problem such that the procedure achieves an approximation factor at least
  \[
    \sqrt{1 + \frac{2}{\Ceil{\frac{1}{\half - \alpha}}-2}} . 
  \]
\item For any space-based procedure, there exists a problem such that the procedure achieves an approximation factor at least
  \[
    \sqrt{1 + \frac1{\Ceil{\frac{1}{\half - \alpha}}^2 - 2\Ceil{\frac{1}{\half - \alpha}}}} .
  \]
\end{itemize}
\end{theorem}

The space-based bound given in \thmref{hilbertlower} is tight for $\alpha \rightarrow 1/2$. This can be seen by noting that the limit of the space-based lower bound for $\alpha \rightarrow 1/2$ is $(\half+\alpha)/\sqrt{2\alpha}$, which is exactly the guarantee provided in \cite{Minsker13} for the space-based geometric median procedure. For smaller $\alpha$, there is a gap between the guarantee of Minsker for the geometric median and our lower bound. 
\vspace{1em}
\begin{proof}
Fix $\alpha$, and let $n = \ceil{\frac{1}{\half - \alpha}}$. Consider the Euclidean space $\reals^{n}$ with $\rho(x,y) = \norm{x -y}$. Let $e_1,\ldots,e_n$ be the standard basis vectors. These are the vertices of a regular simplex with side length $\norm{e_i-e_j} = \sqrt{2}$. Let $b_1,\ldots,b_n$ such that $b_i$ is the center of the hyperface of the simplex opposing $e_i$. Then $\norm{b_i - e_j} = r_{n-1,2}$ for all $j \neq i$, where $r_{n,2} = \sqrt{\frac{n-1}{n}}$ is as defined in \lemref{tetra}.
(see \figref{simplex}). 

Consider $W$ with $k/n$ points at each of $b_1,\ldots,b_n$. Then $\omed_W(e_i,\alpha) \leq \omed_W(e_i,1-\frac{1}{n}) = \norm{e_i-b_j} = r_{n-1,2}$ for any $j \neq i$. Any set-based procedure must select $b_i$ for some $i$. if $w_\opt = e_i$, the resulting approximation factor is $\norm{e_i-b_i}/r_{n-1,2} = \sqrt{\frac{n-2}{n-1}}\norm{e_i-b_i}$. 
For $\norm{b_i - e_i}$, consider for instance $b_1$ and $e_1$. We have $b_1 = (0, \frac{1}{n-1},\ldots,\frac{1}{n-1})$, therefore $\norm{b_1 - e_1} =  \sqrt{\frac{n}{n-1}}$. The approximation factor of the procedure is thus at least $\sqrt{\frac{n}{n-2}}$. 

For a set-based procedure,
whatever $y$ it returns, there exists at least one $i$ such that $\norm{y-a_i} \geq r_{n,2}$. Therefore the approximation factor is at least $r_{n,2}/r_{n-1,2} = \sqrt{\frac{n-1}{n}}/\sqrt{\frac{n-2}{n-1}} = \sqrt{1+\frac{1}{n^2-2n}}$.

\begin{figure}[h]
\begin{center}
\tdplotsetmaincoords{70}{40}
\begin{tikzpicture}[tdplot_main_coords,scale = 3]
\draw[dashed] (0, 0, 0.61)-- (-0.288, -0.5, -0.204) ;
\draw[dashed] (-0.288, 0.5, -0.204) -- (-0.288, -0.5, -0.204) node[anchor=north] {$a_1$};
\draw[dashed] (-0.288, -0.5, -0.204) -- (0.577, 0, -0.204) node[anchor=west] {$a_2$};
\draw[dashed] (0.577, 0, -0.204) -- (-0.288, 0.5, -0.204) node[anchor=west] {$a_3$};
\draw[dashed] (0.577, 0, -0.204) -- (0, 0, 0.61) node[anchor=west] {$a_4$};
\draw[dashed] (-0.288, 0.5, -0.204) -- (0, 0, 0.61);
\fill (0, 0, 0.61)  circle(0.5pt);
\fill (-0.288, 0.5, -0.204)  circle(0.5pt);
\fill (0.577, 0, -0.204)  circle(0.5pt);
\fill (-0.288, -0.5, -0.204)  circle(0.5pt);

\fill  (-0.19,0,0.068) circle(0.5pt) node[anchor=west] {$b_2$};
\end{tikzpicture}
\caption{The regular simplex in $\reals^3$, $n = 4$. $a_i$ is a vertex, $b_i$ is the center of the face opposite $a_i$.}
\label{fig:simplex}
\end{center}
\end{figure}
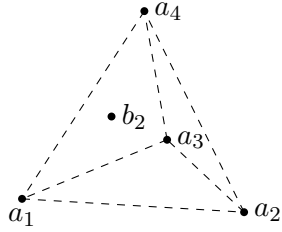

\end{proof}

For space-based procedures, we have seen that while there exists a
lower bound of $2$ for general metric spaces, in a Hilbert space
better approximation factors can be achieved. Is it possible that in
Banach spaces the same approximation factor can also be achieved? The
following theorem shows that the answer is no. 

\begin{theorem}\label{thm:spacebanach}
Let $\alpha = 1/6$. There exists a Banach space for which an approximation factor of $(\half + \alpha)/\sqrt{2\alpha}$ cannot be achieved. 
\end{theorem}
\begin{proof}
Consider the space $\reals^n$ with the distance defined by a $p$-norm. Let $n = 1/(\half - \alpha) = 3$. Construct $W$ as in the proof of \thmref{hilbertlower}, with $k/n$ points in each of $b_1,\ldots,b_n$, where $b_i$ is the center (in the $p$-norm) of the hyperface opposing the basis vector $e_i$. As in the proof of \thmref{hilbertlower}, the approximation factor for any space-based procedure for this problem is at least $r_{n,p}/r_{n-1,p}$. For $p = 3/2$, we have $r_{n,p}/r_{n-1,p} = \frac{2}{5^{1/3}} > \frac{2}{\sqrt{3}} = \frac{\half + \alpha}{\sqrt{2\alpha}}$. 
\end{proof}

\begin{table}[p!]
  \begin{center}
\begin{tabular}{l|l|l|l}
& General Metric & Banach & Hilbert \\
\hline
\hline
Set-based &&& \\
\cline{1-1}
Optimal  & $=3$ & $=3$ & \parbox{10em}{
\begin{align*}
  &\geq \sqrt{1 + \frac{2}{\Ceil{\frac{1}{\half - \alpha}}-2}} \\
&\xrightarrow{\alpha \rightarrow 1/2} 1/\sqrt{2\alpha}
\end{align*}
} \\
Median distance & $=3$ & $=3$ & $=3$ \\
Geometric median & $= 2 + 1/(2\alpha)$ & $= 2 + 1/(2\alpha)$ & Open  \\
\hline
\hline
Space-based &&& \\
\cline{1-1}
Optimal  & $= 2$ & \parbox{7em}{Strictly larger than for Hilbert spaces} & \parbox{10em}{
\begin{align*}
  &\geq \sqrt{1 + \frac1{\Ceil{\frac{1}{\half - \alpha}}^2 -
2\Ceil{\frac{1}{\half - \alpha}}}} \\ 
&\quad\xrightarrow{\alpha \rightarrow 1/2}
\frac{\half + \alpha}{\sqrt{2\alpha}}
\end{align*}}\\
Median distance & $= 2$ & $= 2$ & $= 2$ \\
Geometric median & $= 1 + 1/(2\alpha)$ & $\leq  1 + 1/(2\alpha)$
($\star$) & $\leq (\half + \alpha)/\sqrt{2\alpha}$ ($\star$) \\
\hline
\end{tabular}
\end{center}
\caption{Approximation factors for $\alpha \in (0,1/2)$, based on type
of procedure and type of space. Results marked with ($\star$) are due to \cite{Minsker13}. Equality indicates matching upper and lower bounds.}
\label{table:comparison}
\end{table}

\begin{table}[p!]
\begin{center}
\begin{tabular}{l|l|l|l}
& General Metric & Banach & Hilbert \\
\hline
\hline
Set-based &&& \\
\cline{1-1}
Optimal & $=6$ & $=6$ & $\geq 3.46$ \\ 
Median distance & $=6$ & $=6$ & $=6$ \\
Geometric median & $= 14.92$ & $= 14.92$ & Open  \\
\hline
\hline
Space-based &&& \\
\cline{1-1}
Optimal & $=4$ & Open &$\geq 2.31$\\ 
Median distance & $=4$ & $=4$ & $=4$ \\
Geometric median & $= 11.65$ & $\leq 11.65$ & $\leq 3.33$\\
\hline
\end{tabular}
\caption{Optimal normalized approximation factors based on the values of $C_\alpha$ given in \tabref{comparison}. The value in each case is $\inf_{\alpha \in (0,\half)} \frac{C_\alpha}{(\half-\alpha)}$ for the corresponding $C_\alpha$. All non-integers are rounded to 2 decimal places.}
\label{table:normcomp}
\end{center}
\end{table}

\subsection{Comparison of Selection Procedures}\label{sec:compmedian}

The results provided above are summarized in \tabref{comparison}. 
When comparing different procedures for different values of $\alpha$, it is useful to compare not only the respective approximation factors but also the upper bound that can be obtained for $\omed_W(w_\opt,\alpha)$. Typically, as in the proof of \propref{alg}, this upper bound will stem from first bounding $\E[\rho(w_\opt, w)] \leq \epsilon$, where the expectation is taken over random i.i.d.\ draws of $w$, and then applying Markov's inequality to obtain $\P[\rho(w_\opt,w) \leq \frac{\epsilon}{\half - \alpha}] \geq \half + \alpha$. In the final step Hoeffding's inequality guarantees that if $k$ is large enough, $|\ball(w_\opt, \epsilon/(\half-\alpha)) \cap W|$ approaches $k(\half+\alpha)$. Therefore, for a large $k$ and a procedure for $\alpha$ with an approximation factor $C_\alpha$, the guarantee approaches $\rho(y,w_\opt) \leq \frac{C_\alpha}{(\half-\alpha)}\cdot \epsilon$. For a procedure with an approximation factor $C_\alpha$, we call $\frac{C_\alpha}{(\half-\alpha)}$ the \emph{normalized} approximation factor of the procedure. This is the approximation factor with respect to $\E[\rho(w_\opt,\alpha)]$. When the procedure supports a range of $\alpha$, the optimal normalized factor can be found by minimizing $\frac{C_\alpha}{(\half-\alpha)}$ over $\alpha \in (0,\half)$.
If $C_\alpha =C$ is a constant, the optimal normalized approximation factor is $2C$, achieved when $\alpha = 0$. The optimal normalized approximation factors, based on the known approximation factors as a function of $\alpha$, are given in \tabref{normcomp}.

We observe that for set-based procedures, the median distance is superior to the geometric median for general metric spaces as well as for general Banach spaces. It is an open question whether better results can be achieved for Hilbert spaces using set-based procedures.

For space-based procedures, the median distance is again superior,
except in the case of a Hilbert space, where the geometric median is
superior. The case of a Hilbert space is arguably the most useful in
common applications such as linear regression.
Nevertheless, gaps still remain and it would be interesting to develop
optimal methods.

Implementing the geometric median procedure in a space-based formulation is computationally efficient for Hilbert spaces when accurate distances are available \cite{Minsker13}. However, it is unknown whether and how the procedure can be implemented when only unreliable distance estimations are available, as in \secref{randomdist}. A useful implementation should be both computationally feasible and statistically efficient, while degrading the approximation factors as little as possible.

\section{Predicting Without a Metric on Predictors}\label{sec:output}
The core technique presented above allows selecting a good candidate out of a set that
includes mostly good candidates, in the presence of a metric between candidates. If the final goal is prediction of a scalar label, 
good prediction can still be achieved without access to a metric between candidates, using the following simple procedure:
For every input data point, calculate the prediction of every candidate, and output the median of the predictions. 
This is a straight-forward generalization of voting techniques for
classification such as when using bagging~\citep{Breiman96}.\footnote{Note, however, that the usual implementation of bagging for regression involves averaging over the outputs of the classifiers, and not taking the median.} The following lemma shows that this approach leads to guarantees similar to those achieved by 
\propref{alg}. 

\begin{lemma}
Let $D$, $\ell: \cZ \times \X \rightarrow \reals_+$ and $L:\X \rightarrow \reals_+$ be defined as in \secref{approx}. Assume that $\cZ = \cX \times \cY$,
and there are functions $f:\cX \times \X \rightarrow \reals$ (the prediction function) and $g:\reals \times \reals$ (the link function) such that 
$\ell((\vx,y),\vw) = g(f(\vx,\vw),y)$. Assume that $g$ is convex its first argument.
Suppose that we have $k$ predictors $w_1,\ldots,w_k$ such that for at least $(\half + \gamma)k$ of them,
$L(\vw) \leq \bar{\ell}$. 
For $x \in \cX, y\in \cY$, let $\hat{y}(\vx)$ be the median of $f(\vx,\vw_1),\ldots,f(\vx,\vw_k)$, 
and let $\hat{\ell}(\vx,y) = g(\hat{y}(\vx),y)$. Let $\hat{L} := \E[\hat{\ell}(\hat{y}(\vx))]$. 
Then 
\[
\hat{L} \leq \left(\frac{1}{2\gamma} + 1\right)\bar{\ell}.
\]
\end{lemma}
\begin{proof}
Let $I = \{ i : L(\vw_i) \leq \bar{\ell}\}$. 
Assume without loss of generality that for $i \in [k-1]$, $f(\vx,\vw_i) \leq f(\vx,\vw_{i+1})$.
Let $t \in [k]$ such that $\hat{y}(\vx) = f(\vx,\vw_t)$.
By the convexity of $g$, at least one of $g(f(\vx,\vw_{t}),y) \leq g(f(\vx,\vw_{t-1}, y))$ and $g(f(\vx,\vw_{t}),y) \leq g(f(\vx,\vw_{t+1}, y))$ holds.
assume without loss of generality that the first inequality holds. It follows that for all $i \in [t]$, $g(f(\vx,\vw_i),y) \geq g(f(\vx,\vw_{t}, y))$. 
Therefore,
\begin{align*}
\hat{\ell}(\vx,y) &= g(f(\vx,\vw_{t}), y)) \leq \frac{1}{|I \cap [t]|} \sum_{i \in I \cap [t]} g(f(\vx,\vw_i),y) \\
&\leq \frac{1}{|I \cap [t]|} \sum_{i \in I} g(f(\vx,\vw_i),y) = \frac{1}{|I \cap [t]|} \sum_{i \in I} \ell((\vx,y),\vw_i).
\end{align*}
Taking expectation over $(\vx,y)$, 
\[
\hat{L} \leq \frac{1}{|I \cap [t]|} \sum_{i \in I} L(\vw_i) \leq \frac{|I|}{|I \cap [t]|}\bar{\ell} \leq \frac{\half+\gamma}{\gamma}\bar{\ell},
\]
where the last inequality follows from the assumption that $|I| \geq (\half + \gamma)k$.
\end{proof}

A downside of this approach is that each prediction requires many applications of a predictor. 
If there is also access to unlimited unlabeled data, a possible approach to circumvent this issue is to generate
predictions for a large set of random unlabeled data points based on the aggregate predictor, and then use the resulting labeled pairs as a training set to find a single predictor with a loss that approaches the loss of the aggregate predictor. A similar approach for derandomizing randomized classifiers was suggested by \cite{kaariainen}. 

\section{Conclusion}

In this paper we show several applications of a generalized median-of-means approach to estimation.
In particular, for linear regression we establish convergence rates for heavy-tailed distributions that 
match the min-max rates up to logarithmic factors. We further show conditions that allow parameter estimation
using the Lasso under heavy-tailed noise, and cases under which low-rank covariance matrix approximation is possible
for heavy-tailed distributions. 

The core technique is based on performing independent estimates on separate random samples, and then
combining these estimates. Other works have considered approaches which resemble this general scheme but
provide other types of guarantees.
For instance, in \cite{zhangDuWa13}, faster parallel kernel ridge regression is achieved by performing loss minimizations on independent samples and then averaging the resulting estimators. In \cite{RakhlinSrTs13}, faster rates of convergence for regression for some classes of estimators are achieved, using linear combinations of risk minimizers over subsets of the class of estimators. These works, together with ours, demonstrate that empirical risk minimization can be used as a black box to generate new algorithms with improved statistical performance.

\acks{%
  Part of this work was completed while the authors were at Microsoft Research New England.
  Daniel Hsu was supported by a Yahoo Academic Career Enhancement Award.
  Sivan Sabato is supported by the Lynne and William Frankel Center for Computer Science.}

\appendix

\section{Proof of \thmref{lasso}}\label{app:proofs}

From the definition of $\hat{\vw}$ as a minimizer we have
\begin{equation}\label{eq:psi}
\|\Psi(\vw_\opt - \hat{\vw})\|_2^2 + 2 \lambda \|\hat{\vw}\|_1
\leq 2 \lambda \|\vw_\opt\|_1 + 2\veps^\t \Psi(\hat{\vw} - \vw_\opt) . 
\end{equation}
By H\"older's inequality the assumptions of the theorem,
$2\veps^\t \Psi(\hat{\vw} - \vw_\opt)
\leq 2\|\veps^\t \Psi\|_\infty \|\hat{\vw} - \vw_\opt\|_1
\leq \lambda \|\hat{\vw} - \vw_\opt\|_1$.
Combining this with \eqref{psi} gives
\[
\|\Psi(\vw_\opt - \hat{\vw})\|_2^2
\leq
2 \lambda \|\vw_\opt\|_1
- 2 \lambda \| \hat{\vw}\|_1
+ \lambda \|\hat{\vw} - \vw_\opt\|_1
.
\]
Adding $\lambda  \|(\hat{\vw} - \vw)\|_1$ to both sides we get
\begin{align*}
\|\Psi(\vw_\opt - \hat{\vw})\|_2^2 + \lambda   \| \hat{\vw} - \vw_\opt\|_1
& \leq 2 \lambda
\Bigl( \| \hat{\vw} - \vw_\opt\|_1 + \| \vw_\opt\|_1 - \| \hat{\vw}\|_1 \Bigr) \\
& = 2 \lambda \sum_{j=1}^d \Bigl(|\hat{\vw}[j] - \vw_\opt[j]| + |\vw_\opt[j]| - |\hat{\vw}[j]| \Bigr) \\
& = 2 \lambda \sum_{j\in\supp(\vw)}
\Bigl( |\hat{\vw}[j] - \vw_\opt[j]| + |\vw_\opt[j]| - |\hat{\vw}[j]| \Bigr) \\
& \leq 4 \lambda \sum_{j\in\supp(\vw)} |\hat{\vw}[j] - \vw_\opt[j]| \\
& = 4 \lambda \| [\hat{\vw} - \vw_\opt]_{\supp(\vw)}\|_1.
\end{align*}
It follows that
\[
\norm{[\hat{\vw} - \vw_\opt]_{\supp(\vw_\opt)^C}}_1 \leq 3\norm{[\hat{\vw}-\vw_\opt]_{\supp(\vw_\opt)}},
\]
therefore $\hat{\vw}-\vw_\opt \in E_s$. Denote $\vec{\delta} = \hat{\vw} - \vw$.
The above derivation also implies 
\[
\norm{\Psi \vec{\delta}}_2^2 \leq 3\lambda \norm{[\vec{\delta}]_{\supp(\vw_\opt)}}_1 \leq 
3\lambda \norm{\vec{\delta}_{[s]}}_1 \leq 3\lambda\sqrt{s} \norm{\vec{\delta}_{[s]}}_2.
\]
Denote for brevity $\gamma = \gamma(\Psi,s)$.
From the definition of $\gamma$,
\[
\norm{\vec{\delta}_{[s]}}^2_2 \leq \frac{1}{\gamma^2}\norm{\Psi \vec{\delta}}_2^2 \leq \frac{3\lambda\sqrt{s} \norm{\vec{\delta}_{[s]}}_2}{\gamma^2} ,
\]
Therefore
$\norm{\vec{\delta}_{[s]}}_2 \leq \frac{3\lambda\sqrt{s}}{\gamma^2}$.
Now, 
\[
\norm{\vec{\delta}}_2 = \norm{\vec{\delta}_{[s]^C}}_2 + \norm{\vec{\delta}_{[s]}}_2 \leq  \sqrt{\norm{\vec{\delta}_{[s]^C}}_1\norm{\vec{\delta}_{[s]^C}}_\infty} + \norm{\vec{\delta}_{[s]}}_2.
\]

From $\vec{\delta} \in E_s$ we get $\norm{\vec{\delta}_{[s]^C}}_1 \leq 3\norm{\vec{\delta}_{[s]}}_1$.
In addition, since $\vec{\delta}_{[s]}$ spans the largest coordinates of $\vec{\delta}$ in absolute value, $\norm{\vec{\delta}_{[s]^C}}_\infty \leq \norm{\vec{\delta}_{[s]}}_1/s$. Combining these with the inequality above we get
\[
\norm{\vec{\delta}}_2 \leq 3\norm{\vec{\delta}_{[s]}}_1/\sqrt{s} + \norm{\vec{\delta}_{[s]}}_2 \leq 4\norm{\vec{\delta}_{[s]}}_2 \leq \frac{12\lambda \sqrt{s}}{\gamma^2}.
\]
\hfill\BlackBox

\bibliography{bib}

\begin{thebibliography}{45}
\providecommand{\natexlab}[1]{#1}
\providecommand{\url}[1]{\texttt{#1}}
\expandafter\ifx\csname urlstyle\endcsname\relax
  \providecommand{\doi}[1]{doi: #1}\else
  \providecommand{\doi}{doi: \begingroup \urlstyle{rm}\Url}\fi

\bibitem[Alon et~al.(1999)Alon, Matias, and Szegedy]{alon99}
Noga Alon, Yossi Matias, and Mario Szegedy.
\newblock The space complexity of approximating the frequency moments.
\newblock \emph{Journal of Computer and System Sciences}, 58:\penalty0
  137--147, 1999.

\bibitem[Audibert and Catoni(2011)]{AC11}
Jean-Yves Audibert and Olivier Catoni.
\newblock Robust linear least squares regression.
\newblock \emph{Ann. Stat.}, 39\penalty0 (5):\penalty0 2766--2794, 2011.

\bibitem[Belloni and Chernozhukov(2011)]{belloni2011l1}
Alexandre Belloni and Victor Chernozhukov.
\newblock l1-penalized quantile regression in high-dimensional sparse models.
\newblock \emph{The Annals of Statistics}, 39\penalty0 (1):\penalty0 82--130,
  2011.

\bibitem[Bickel et~al.(2009)Bickel, Ritov, and Tsybakov]{bickel09}
Peter~J Bickel, Ya'acov Ritov, and Alexandre~B Tsybakov.
\newblock Simultaneous analysis of lasso and dantzig selector.
\newblock \emph{The Annals of Statistics}, 37\penalty0 (4):\penalty0
  1705--1732, 2009.

\bibitem[Breiman(1996)]{Breiman96}
Leo Breiman.
\newblock Bagging predictors.
\newblock \emph{Machine learning}, 24\penalty0 (2):\penalty0 123--140, 1996.

\bibitem[{Brownlees} et~al.(2014){Brownlees}, {Joly}, and {Lugosi}]{lugosi}
C.~{Brownlees}, E.~{Joly}, and G.~{Lugosi}.
\newblock {Empirical risk minimization for heavy-tailed losses}.
\newblock \emph{ArXiv e-prints}, June 2014.

\bibitem[Bubeck et~al.(2013)Bubeck, Cesa-Bianchi, and Lugosi]{BubeckCeLu13}
S.~Bubeck, N.~Cesa-Bianchi, and G.~Lugosi.
\newblock Bandits with heavy tail.
\newblock \emph{IEEE Transactions on Information Theory}, 59:\penalty0
  7711--7717, 2013.

\bibitem[Catoni(2012)]{catoni}
Olivier Catoni.
\newblock Challenging the empirical mean and empirical variance: a deviation
  study.
\newblock \emph{Ann. Inst. H. Poincaré Probab. Statist.}, 48\penalty0
  (4):\penalty0 1148--1185, 2012.

\bibitem[Chatterjee and Lahiri(2013)]{chatterjee2013rates}
A~Chatterjee and SN~Lahiri.
\newblock Rates of convergence of the adaptive lasso estimators to the oracle
  distribution and higher order refinements by the bootstrap.
\newblock \emph{The Annals of Statistics}, 41\penalty0 (3):\penalty0
  1232--1259, 2013.

\bibitem[Chatterjee and Hadi(1986)]{chatterjee1986influential}
Samprit Chatterjee and Ali~S Hadi.
\newblock Influential observations, high leverage points, and outliers in
  linear regression.
\newblock \emph{Statistical Science}, 1\penalty0 (3):\penalty0 379--393, 1986.

\bibitem[Efron(1979)]{Efron79}
Bradley Efron.
\newblock Bootstrap methods: another look at the jackknife.
\newblock \emph{The annals of Statistics}, pages 1--26, 1979.

\bibitem[Fan et~al.(2012)Fan, Fan, and Barut]{fan2012adaptive}
Jianqing Fan, Yingying Fan, and Emre Barut.
\newblock Adaptive robust variable selection.
\newblock \emph{arXiv preprint arXiv:1205.4795}, 2012.

\bibitem[Hsu and Sabato(2013)]{HS13-heavy}
Daniel Hsu and Sivan Sabato.
\newblock Approximate loss minimization with heavy tails.
\newblock \emph{CoRR}, abs/1307.1827, 2013.
\newblock URL \url{http://arxiv.org/abs/1307.1827}.

\bibitem[Hsu and Sabato(2014)]{HsuSabato14}
Daniel Hsu and Sivan Sabato.
\newblock Heavy-tailed regression with a generalized median-of-means.
\newblock In \emph{Thirty-First International Conference on Machine Learning},
  2014.

\bibitem[Hsu et~al.(2014)Hsu, Kakade, and Zhang]{HKZ12}
Daniel Hsu, Sham~M. Kakade, and Tong Zhang.
\newblock Random design analysis of ridge regression.
\newblock \emph{Foundations of Computational Mathematics}, 14\penalty0
  (3):\penalty0 569--600, 2014.

\bibitem[Huber(1981)]{Huber81}
P.~J. Huber.
\newblock \emph{Robust Statistics}.
\newblock Wiley, 1981.

\bibitem[Juditsky and Nemirovski(2008)]{JN08}
Anatoli Juditsky and Arkadii~S. Nemirovski.
\newblock Large deviations of vector-valued martingales in 2-smooth normed
  spaces.
\newblock \emph{ArXiv e-prints}, 0809.0813, 2008.

\bibitem[K{\"a}{\"a}ri{\"a}inen(2005)]{kaariainen}
Matti K{\"a}{\"a}ri{\"a}inen.
\newblock Generalization error bounds using unlabeled data.
\newblock In \emph{Learning Theory}, pages 127--142. Springer, 2005.

\bibitem[Koltchinskii et~al.(2011)Koltchinskii, Lounici, and Tsybakov]{KLT11}
V.~Koltchinskii, K.~Lounici, and A.~B. Tsybakov.
\newblock Nuclear norm penalization and optimal rates for noisy low rank matrix
  completion.
\newblock \emph{Annals of Statistics}, 39\penalty0 (5):\penalty0 2302--2329,
  2011.

\bibitem[Kuratowski(1935)]{Kuratowski35}
Casimir Kuratowski.
\newblock Quelques probl\`{e}mes concernant les espaces m\'{e}triques
  non-s\'{e}parables.
\newblock \emph{Fundamenta Mathematicae}, 25\penalty0 (1):\penalty0 534--545,
  1935.

\bibitem[Lepski(1991)]{lepski}
O.~V. Lepski.
\newblock Asymptotically minimax adaptive estimation {I}: Upper bounds.
  optimally adaptive estimates.
\newblock \emph{Theory Probab. Appl.}, 36\penalty0 (4):\penalty0 682--697,
  1991.

\bibitem[{Lerasle} and {Oliveira}(2011)]{LO11}
M.~{Lerasle} and R.~I. {Oliveira}.
\newblock {Robust empirical mean Estimators}.
\newblock \emph{ArXiv e-prints}, December 2011.

\bibitem[Levin(2005)]{Levin-notes}
Leonid~A. Levin.
\newblock Notes for miscellaneous lectures.
\newblock \emph{CoRR}, abs/cs/0503039, 2005.

\bibitem[Litvak et~al.(2005)Litvak, Pajor, Rudelson, and
  Tomczak-Jaegermann]{LPRTJ05}
Alexander~E. Litvak, Alain Pajor, Mark Rudelson, and Nicole Tomczak-Jaegermann.
\newblock Smallest singular value of random matrices and geometry of random
  polytopes.
\newblock \emph{Adv. Math.}, 195\penalty0 (2):\penalty0 491--523, 2005.
\newblock ISSN 0001-8708.
\newblock \doi{10.1016/j.aim.2004.08.004}.
\newblock URL \url{http://dx.doi.org/10.1016/j.aim.2004.08.004}.

\bibitem[Mahdavi and Jin(2013)]{MJ13}
Mehrdad Mahdavi and Rong Jin.
\newblock Passive learning with target risk.
\newblock In \emph{Twenty-Sixth Conference on Learning Theory}, 2013.

\bibitem[{Mendelson}(2014)]{Mendelson14}
S.~{Mendelson}.
\newblock {Learning without Concentration}.
\newblock \emph{ArXiv e-prints}, January 2014.

\bibitem[Minsker(2013)]{Minsker13}
Stanislav Minsker.
\newblock Geometric median and robust estimation in banach spaces.
\newblock \emph{arXiv preprint arXiv:1308.1334}, 2013.

\bibitem[Nemirovsky and Yudin(1983)]{NY}
A.~S. Nemirovsky and D.~B. Yudin.
\newblock \emph{Problem Complexity and Method Efficiency in Optimization}.
\newblock Wiley-Interscience, 1983.

\bibitem[Nussbaum(1999)]{Nussbaum99}
M.~Nussbaum.
\newblock Minimax risk: Pinsker bound.
\newblock In S.~Kotz, editor, \emph{Encyclopedia of Statistical Sciences,
  Update Volume 3}, pages 451--460. Wiley, New York, 1999.

\bibitem[Oliveira(2010)]{Oliveria10}
Roberto Oliveira.
\newblock Sums of random {H}ermitian matrices and an inequality by {R}udelson.
\newblock \emph{Electron. Commun. Probab.}, 15\penalty0 (19):\penalty0
  203--212, 2010.

\bibitem[Rakhlin et~al.(2013)Rakhlin, Sridharan, and Tsybakov]{RakhlinSrTs13}
Alexander Rakhlin, Karthik Sridharan, and Alexandre~B. Tsybakov.
\newblock Empirical entropy, minimax regret and minimax risk.
\newblock \emph{arXiv preprint arXiv:1308.1147}, 2013.

\bibitem[{Shamir}(2014)]{ohad}
O.~{Shamir}.
\newblock {The Sample Complexity of Learning Linear Predictors with the Squared
  Loss}.
\newblock \emph{ArXiv e-prints}, June 2014.

\bibitem[Srebro et~al.(2010)Srebro, Sridharan, and Tewari]{smooth-loss}
Nathan Srebro, Karthik Sridharan, and Ambuj Tewari.
\newblock Smoothness, low noise and fast rates.
\newblock In \emph{Advances in Neural Information Processing Systems 23}, 2010.

\bibitem[Srivastava and Vershynin(2013)]{SV}
N.~Srivastava and R.~Vershynin.
\newblock Covariance estimation for distributions with $2+\epsilon$ moments.
\newblock \emph{Annals of Probability}, 41:\penalty0 3081--3111, 2013.

\bibitem[Tibshirani(1996)]{tibshirani96}
Robert Tibshirani.
\newblock Regression shrinkage and selection via the lasso.
\newblock \emph{Journal of the Royal Statistical Society. Series B
  (Methodological)}, pages 267--288, 1996.

\bibitem[van~de Geer and M{\"u}ller(2012)]{van2012quasi}
Sara van~de Geer and Patric M{\"u}ller.
\newblock Quasi-likelihood and/or robust estimation in high dimensions.
\newblock \emph{Statistical Science}, 27\penalty0 (4):\penalty0 469--480, 2012.

\bibitem[Wang et~al.(2007)Wang, Li, and Jiang]{wang2007robust}
Hansheng Wang, Guodong Li, and Guohua Jiang.
\newblock Robust regression shrinkage and consistent variable selection through
  the lad-lasso.
\newblock \emph{Journal of Business \& Economic Statistics}, 25\penalty0
  (3):\penalty0 347--355, 2007.

\bibitem[Wang(2013)]{wang2013l1}
Lie Wang.
\newblock L1 penalized lad estimator for high dimensional linear regression.
\newblock \emph{Journal of Multivariate Analysis}, 2013.

\bibitem[Wolfowitz(1950)]{wolfowitz50}
J.~Wolfowitz.
\newblock Minimax estimates of the mean of a normal distribution with known
  variance.
\newblock \emph{The Annals of Mathematical Statistics}, 21:\penalty0 218--230,
  1950.

\bibitem[Wu and Liu(2009)]{wu2009variable}
Yichao Wu and Yufeng Liu.
\newblock Variable selection in quantile regression.
\newblock \emph{Statistica Sinica}, 19\penalty0 (2):\penalty0 801, 2009.

\bibitem[Zhang(2009)]{zhang09}
Tong Zhang.
\newblock Some sharp performance bounds for least squares regression with l1
  regularization.
\newblock \emph{The Annals of Statistics}, 37\penalty0 (5A):\penalty0
  2109--2144, 2009.

\bibitem[Zhang et~al.(2013)Zhang, Duchi, and Wainwright]{zhangDuWa13}
Yuchen Zhang, John~C Duchi, and Martin~J Wainwright.
\newblock Divide and conquer kernel ridge regression: A distributed algorithm
  with minimax optimal rates.
\newblock \emph{arXiv preprint arXiv:1305.5029}, 2013.

\bibitem[Zhou(2009)]{Zhou09}
Shuheng Zhou.
\newblock Restricted eigenvalue conditions on subgaussian random matrices.
\newblock \emph{arXiv preprint arXiv:0912.4045}, 2009.

\bibitem[Zou(2006)]{zou2006adaptive}
Hui Zou.
\newblock The adaptive lasso and its oracle properties.
\newblock \emph{Journal of the American statistical association}, 101\penalty0
  (476):\penalty0 1418--1429, 2006.

\bibitem[Zou and Yuan(2008)]{zou2008composite}
Hui Zou and Ming Yuan.
\newblock Composite quantile regression and the oracle model selection theory.
\newblock \emph{The Annals of Statistics}, 36\penalty0 (3):\penalty0
  1108--1126, 2008.

\end{thebibliography}

\end{document}